\documentclass[twoside,11pt]{article}
\usepackage{jmlr2e}

\usepackage{rotate}

\usepackage{amsmath}
\usepackage{amssymb}
\usepackage{bm}
\usepackage{enumitem}
\usepackage{tikz}
\usepackage{tikz-3dplot} 
\usetikzlibrary{shapes}
\usepackage{url}

\usepackage{algorithm}
\usepackage[noend]{algpseudocode}

\usepackage{mathpartir}
\usepackage{myabbrevmacros}
\usepackage{myutilmacros}

\newcommand{\inner}[2]{\ensuremath{\langle #1, #2\rangle}}

\newcommand{\two}{\ensuremath{\mathbf{2}}}

\newcommand{\distfunc}{\ensuremath{\mathbf{Dist}}}
\newcommand{\setfunc}{\ensuremath{\mathbf{Set}}}
\newcommand{\seqfunc}{\ensuremath{\mathbf{Seq}}}
\newcommand{\strfunc}{\ensuremath{\mathbf{Str}}}
\newcommand{\permfunc}{\ensuremath{\mathbf{Perm}}}

\newcommand{\mvconst}{\ensuremath{\delta}}
\newcommand{\consts}{\ensuremath{\Delta}}
\newcommand{\mvaconst}{\ensuremath{\gamma}}
\newcommand{\aconsts}{\ensuremath{\Gamma}}
\newcommand{\mvsconst}{\ensuremath{\sigma}}

\newcommand{\expansion}{\ensuremath{\geq}}
\newcommand{\expandsinto}{\ensuremath{\leq}}
\newcommand{\dnf}{\ensuremath{\operatorname{dnf}}}
\newcommand{\expand}{\ensuremath{\operatorname{expand}}}

\newcommand{\belief}{\ensuremath{\mathbb{B}}}
\newcommand{\mvhintree}{\ensuremath{\mathbb{H}}}

\newcommand{\renorm}{\ensuremath{\operatorname{renorm}}}
\newcommand{\suppo}{\ensuremath{\operatorname{supp}}^0}

\newcommand{\supp}{\ensuremath{\operatorname{supp}}}

\newcommand{\mvconj}{\ensuremath{\operatorname{conj}}}
\newcommand{\mvhyp}{\ensuremath{\mathbb{D}}}

\newcommand{\prove}{\ensuremath{\operatorname{prove}}}

\newcommand{\mvpath}{\rho}

\ShortHeadings{On Learning to Prove}{Huang}
\firstpageno{1}

\begin{document}

\title{On Learning to Prove}

\author{\name Daniel Huang \email dehuang@berkeley.edu \\
       \addr Department of Electrical Engineering and Computer Science\\
       University of California\\
       Berkeley, CA, USA}


\maketitle

\begin{abstract}
In this paper, we consider the problem of learning a first-order theorem prover that uses a representation of beliefs in mathematical claims to construct proofs. The inspiration for doing so comes from the practices of human mathematicians where ``plausible reasoning" is applied in addition to deductive reasoning to find proofs.

Towards this end, we introduce a representation of beliefs that assigns probabilities to the exhaustive and mutually exclusive first-order \emph{possibilities} found in Hintikka's theory of distributive normal forms. The representation supports Bayesian update, induces a distribution on statements that does not enforce that logically equivalent statements are assigned the same probability, and suggests an embedding of statements into an associated Hilbert space.

We then examine \emph{conjecturing} as model selection and an \emph{alternating-turn} game of determining consistency. The game is amenable (in principle) to \emph{self-play} training to learn beliefs and derive a prover that is complete when \emph{logical omniscience} is attained and sound when beliefs are \emph{reasonable}. The representation has super-exponential space requirements as a function of quantifier depth so the ideas in this paper should be taken as theoretical. We will comment on how \emph{abstractions} can be used to control the space requirements at the cost of completeness.
\end{abstract}


\section{Introduction}
\label{sec:intro}

The process of discovering a mathematical proof can be seen as a perfect information game where the goal is to show that a path exists (\ie, the proof) between a given starting state (\ie, the axioms) and ending state (\ie, the claim) using a predefined collection of rules (\ie, deduction). Like other perfect information games such as Go and Chess, the complexity of the theorem proving game involves managing the combinatorial nature of the search space. We can do this, for instance, by identifying useful heuristics and patterns. This is one sense in which players can learn and improve from their experiences playing the game.

The idea of ``learning from experience" suggests that we can apply \emph{machine learning} to learn these heuristics and patterns as opposed to distilling them manually from human experience. Towards this end, researchers have demonstrated that machine learned algorithms can navigate the search spaces of Go~\citep{silver2016mastering} and Chess~\citep{silver2017mastering} at a level exceeding human experts (\ie, consistently defeat the best human players). Researchers have also experimented with applying machine learning to theorem provers~\citep[\eg, see ][]{komendantskaya2012machine,kaliszyk2014machine,gauthier2017tactictoe,duncantactics,neurosat,holstep,premise,guided,mctsrl,huang2018gamepad}, although the problem is much more difficult compared to Go and Chess when quantifiers are involved.\footnote{The state spaces of Chess and Go, albeit large, are finite. In contrast, quantifiers can range over infinite domains.}

In this paper, we consider the problem of learning a prover for \emph{first-order logic},\footnote{First-order logic along with the axioms of set theory are expressive---they are in principle sufficient to encode most of modern mathematics, although humans generally work at a higher level of abstraction and within a natural language extended with mathematical concepts as opposed to a formal language.} a well-understood setting with quantification, where we directly use a representation of beliefs in mathematical claims to construct proofs.\footnote{The literature on automated theorem proving is expansive see~\citep[\eg, see][for a survey of first-order methods]{fitting2012first}. Most provers use a proof-theoretic system as the primary abstraction for representing mathematical knowledge.} The inspiration for doing so comes from the practices of human mathematicians where ``plausible reasoning"\footnote{P{\'o}lya has written extensively on plausible reasoning, \ie, the heuristic and non-deductive aspects of mathematical reasoning, including (1) weighing evidence for and against a conjecture, (2) making physical analogies, and (3) reasoning from randomness~\citep[\eg, see][]{polya1990mathematics1,polya1990mathematics2,polya2004solve}.} is used in addition to deductive reasoning to discover proofs.\footnote{The non-deductive aspects of mathematical reasoning has been recognized by mathematicians and philosophers~\citep[\eg, see][]{hacking1967slightly,corfield2003towards,parikh2010sentences,seidenfeld2012uncertainty,mazur2014plausible}.}

We start by introducing a representation of beliefs that assigns probabilities to the exhaustive and mutually exclusive first-order \emph{possibilities} found in the theory of first-order \emph{distributive normal forms} (dnfs) devised by the philosopher Jaakko Hintikka (Section~\ref{sec:repr}). The idea of assigning weights to dnfs has been proposed by~\citet{hintikka1970surface} in the context of inductive philosophy so the idea is not new. Our contribution here is extract and formalize some of these ideas for the purposes of ``learning to prove". We show that the representation supports a form of Bayesian update and induces a distribution on the validity of first-order statements that does not enforce that logically equivalent statements are assigned the same probability---otherwise, we would end up in a circular situation where we require a prover in order to assign probabilities. In addition, we show that there is an \emph{embedding} of first-order statements into an associated Hilbert space where mutual exclusion in logic translates into orthogonality in the space.

Next, we consider two applications that a direct probabilistic representation of beliefs in mathematical claims has for ``learning to prove". First, we identify \emph{conjecturing} as a form of (statistical) \emph{model selection} (Section~\ref{sec:tp}). Second, we introduce an \emph{alternating-turn game} that involves determining the consistency of possibilities (Section~\ref{sec:game}). The game is amenable (in principle) to self-play training, a technique that has demonstrated success in learning expert-level play for the games of Go and Chess, to learn beliefs that can be used to construct a prover that is complete when \emph{logical omniscience}\footnote{An agent is \emph{logically omniscient} if it knows all the logical consequences that follow from a set of axioms. Consequently, logical omniscience should fail in the interim while learning a prover---there is nothing to learn if an agent already possesses knowledge of all theorems.} is attained and sound provided that players maintain \emph{reasonable}\footnote{Roughly speaking, an agent is \emph{reasonable} if it does not assign zero probability to a possibility that it has not been able to falsify. We will define this formally in Section~\ref{subsubsec:repr:ht:renorm}.} beliefs. Implementing and empirically testing self-play for these games is technically challenging and beyond the scope of this paper.

The ideas in this paper should be taken with one major caveat: the space complexity of the representation is (highly) super-exponential as a function of quantifier \emph{depth} (\ie, the maximal number of nested quantifiers) so that the ideas are not practically implementable without modification. Thus our analysis in its current form should only be seen as conducting a thought experiment. As a step towards making the ideas here more practical, we will comment on how to control the sizes of the representations at the cost of completeness by treating certain combinations of properties as observationally indistinguishable, \ie, by making \emph{abstractions} and \emph{lazily} considering more properties as needed (Section~\ref{sec:abs}). This suggests a path towards implementation (\eg, for the game).

As one final qualification concerning the ideas in this paper, we acknowledge that we have taken a somewhat narrow view of ``learning to prove". First, we restrict ourselves to a first-order axiomatic view of mathematics.\footnote{Although the axiomatic approach to mathematics is widely adopted, mathematicians typically do not carry out the paradigm to its full extent and write completely formal proofs. When they do, there are a variety of formal languages they can choose from in addition to first-order logic including higher-order logic and type theories. The practicality of formalizing mathematics has been aided by the development of tools called \emph{interactive theorem provers}. (For instance, \citet{gonthier2013machine} formalized the Feit-Thompson theorem, a deep result in group theory, using an interactive theorem prover.) There are interactive theorem provers based on first-order logic~\citep[\eg, see][]{mizar}, higher-order logic~\citep[\eg, see][]{isabelle}, and type theories~\citep[\eg, see][]{coq}. An interesting direction of future work would be to see how the ideas in this paper apply to higher-order and type-theoretic settings.} Second, we consider only a probabilistic aspect of plausible reasoning.\footnote{The use of probabilistic reasoning to model plausible reasoning is not a new idea---for instance, see work on probabilistic graphical models~\citep{pearl1988pgm} and work on inductive inference~\citep[\eg, see][]{solomonoff1964formal,solomonoff1964formal2,jaeger2005logic}. The field of \emph{automated reasoning}~\citep[\eg, see][for a survey]{robinson2001handbook,robinson2001bhandbook} contains work on other forms of non-deductive reasoning including reasoning by induction~\citep[\eg, see][]{quinlan1986induction,bundy2001ind,comon2001ind}, abduction~\citep[\eg, see][]{console1991relationship,mayer1993first,gabbay1998handbook,denecker2002abduction}, and analogy~\citep[\eg, see][]{davies1987logical,ashley1988modelling,russell1988analogy}.} Finally, we emphasize that our work is not human-style theorem proving~\citep[\eg, see][]{ganesalingam2017fully} even though we take inspiration from human mathematicians. In spite of these limitations and shortcomings, we believe that the ideas presented here offer a descriptive account of ``learning to prove" that cohesively accounts for the role of beliefs in the proving process, the utility of conjecturing, and the value of abstraction.

\section{Preliminaries}
\label{sec:prelim}

We begin by setting up the notation and terminology we will use throughout this paper (Section~\ref{subsec:prelim:notation}). Next, we provide intuition for Hintikka's dnfs (Section~\ref{subsec:prelim:mot}) and then introduce them formally for first-order logic without equality\footnote{The restriction to first-order logic without equality is for simplicity: dnfs are defined for first-order logic with equality as well. All results given here apply to dnfs in both cases with the appropriate modifications. The difference between the two is between an \emph{inclusive} treatment of quantifiers (without equality) and an \emph{exclusive} treatment of quantifiers (with equality). As usual, note that we can include a binary predicate that encodes equality in first-order logic without equality, the difference with the case of first-order logic with equality being that structures may not necessarily be \emph{normal}.} (Section~\ref{subsec:prelim:dnf}). For more background on dnfs, we refer the reader to~\citep{hintikka1965distributive,hintikka1973logic,nelte1997formulas}.

\subsection{Notation and Background}
\label{subsec:prelim:notation}

Let $\two \eqdef \set{0, 1}$. We will interchangeably use $\false$ for $0$ (false) and $\true$ for $1$ (true). $\N$ denotes the set of naturals and $\N^+$ denotes the set of positive naturals. $\R$ denotes the set of reals and $\R^+$ denotes the set of positive reals.

We write $\setfunc(X) \cong X \rightarrow \two$ to indicate the power set of $X$. We write $|X|$ for the cardinality of the set $X$. We will often write a binary relation such as $\sim: X \times X \rightarrow \two$ in infix notation as $x \sim y \eqdef \sim(x, y)$ for $x \in X$ and $y \in X$. The notation $\set{x \ST P(x)}$ where $P: X \rightarrow \two$ is a predicate on a set $X$ indicates a set comprehension. We also write set comprehensions for indexed sets as $\set{x_i \ST P(i)}$ where $P: I \rightarrow \two$ is a predicate on an index set $I$ that indexes $X$.

When order matters, we use $\seq{\cdot}$ for sequences instead of $\set{\cdot}$. We write $\seqfunc^n(X) \eqdef \set{\seq{x \ST x \in X} \ST |\seq{x \ST x \in X}| = n}$ for the set of length $n$ sequences comprised of elements from $X$. We write $\strfunc^n(X)$ for the set of length $n$ strings comprised of elements from $X$.

We will use ellipsis notation ``$\dots$" frequently in this paper. As usual, it means ``fill in the dots with all the missing elements in between". For example, $x_1, \dots, x_n$ gives the elements $x_1$, $x_2$, and so on until $x_n$. When the commas are omitted as in $x_1 \dots x_n$, the notation indicates a string of those elements instead.

\subsubsection{First-order logic}

The syntax of first-order logic (without equality) is summarized below.
\begin{align*}
    \mvterm & \eqdef x \bnfsep c \bnfsep f^n(\mvterm_1, \dots, \mvterm_n) \bnfsep \mvpred^n(\mvterm_1, \dots, \mvterm_n) \\
    \mvform & \eqdef \mvterm \bnfsep \lnot \mvform \bnfsep \mvform \lor \mvform \bnfsep (\exists x) \mvform
\end{align*}
We use the meta-variable $\mvterm$ to refer to terms. A \emph{term} is either a variable $x$, a constant $c$, a $n$-ary function $f^n(\mvterm_1, \dots, \mvterm_n)$ applied to $n$ terms, or a $n$-ary predicate $\mvpred^n(\mvterm_1, \dots, \mvterm_n)$ on $n$ terms. We use the meta-variable $\mvform$ to refer to formulas. A \emph{formula} is either a term ($\mvterm$), the logical negation of a formula ($\lnot \mvform$), the logical \emph{or} of two formulas ($\mvform \lor \mvform$), or an existential quantification ($(\exists x) \mvform$). As usual, we encode logical \emph{and} as $\mvform_1 \land \mvform_2 \eqdef \lnot (\lnot \mvform_1 \lor \lnot \mvform_2)$ and universal quantification as $(\forall x) \mvform \eqdef \lnot (\exists x) \lnot \mvform$ where we assume the usual precedence and use additional (meta-level) parentheses to aid the parsing of formulas. The meta-level notation $(\pm)^b \mvform$ where $b \in \two$ either negates the formula $(\pm)^0 \mvform \eqdef \lnot \mvform$ or leaves it alone $(\pm)^1 \mvform \eqdef \mvform$.

We write a formula with free variables as $\mvform[x_{i_1}, \dots, x_{i_n}]$ where $x_1, x_2, \dots$ is a supply of free variables. A formula without free variables is called a \emph{sentence}. 

We use a standard deductive system for first-order logic and write $\mvform_1 \vdash \mvform_2$ if there is a derivation of $\mvform_2$ using $\mvform_1$, any logical axioms, and the rules of inference. We write $\Gamma = \set{\mvform_1, \dots, \mvform_n}$ to be a set of sentences. We say that $\Gamma$ is \emph{consistent} if a contradiction is not derivable, \ie, both $\Gamma \vdash \mvform$ and $\Gamma \vdash \lnot \mvform$ are not derivable for any $\mvform$ where we take the conjunction of all sentences in $\Gamma$ when it appears to the left of $\vdash$.

We use the standard semantics of first-order logic based on structures.\footnote{For more background on first-order logic, we refer the reader to~\citep{hodges2001elementary}.} A \emph{structure} is a tuple $\cM \eqdef (D, \Sigma, \denote{\cdot})$ where $D$ is a (potentially empty) set called the domain, $\Sigma$ is a signature (the functions and relations of the language), and $\denote{\cdot}$ is an interpretation of the signature. Note that an empty domain cannot be used to interpret a language with constants. We say that a formula $\mvform$ is \emph{satisfiable} in a structure $\cM$ if $\cM \vDash \mvform[a_{i_1}, \dots, a_{i_n}]$ for every $a_{i_1}, \dots, a_{i_n} \in D$ where $\vDash$ is the usual satisfaction relation defined by induction on the structure of formulas and we overload $\mvform[a_{i_1}, \dots, a_{i_n}]$ to mean that the interpretation of the variable $x_{i_m}$ in $\mvform[x_{i_1}, \dots, x_{i_m}, \dots, x_{i_n}]$ is $a_{i_m}$. A sentence is \emph{satisfiable} if there is some structure $\cM$ such that $\cM \vDash \mvform$.

Recall that first-order logic with a standard proof system and standard semantics is \emph{sound} (\ie, $\vDash \mvform$ if $\vdash \mvform$) and \emph{complete} (\ie, $\vdash \mvform$ if $\vDash \mvform$). Thus a sentence $\mvform$ is consistent iff it is satisfiable. A sentence $\mvform$ is inconsistent if it is satisfiable in no structures, consistent if it is satisfiable in at least one structure, and \emph{logically valid} if it is satisfiable in every structure. We write $\mvform_1 \equiv \mvform_2$ when $\mvform_1$ and $\mvform_2$ are \emph{logically equivalent}.

\subsubsection{Graphs and trees}

A \emph{directed graph} is a tuple $(V, E)$ where $V$ is a set of \emph{vertices} and $E \subseteq \set{ (v_1, v_2) \ST v_1, v_2 \in V}$ is a set of \emph{edges}. Because we only consider directed graphs in this paper, we will abbreviate directed graph as \emph{graph}. A \emph{path} in a graph $(V, E)$ is a graph $(V', E')$ of the form $V' \eqdef \set{v_1, \dots, v_k} \subseteq V$ and $E' \eqdef \set{(v_1, v_2), \dots, (v_{k-1}, v_k)} \subseteq E$ where all $v_i$ are distinct. We refer to $v_1$ and $v_k$ as the \emph{endpoints} of the path.

A \emph{(rooted) tree} is a tuple $(V, E, v_R)$ where $(V, E)$ is a graph such that any two vertices are connected by a unique path and $v_R \in V$ is a vertex designated as a \emph{root}. Because there is only one path between any two vertices, a path between $v_1 \in V$ and $v_k \in V$ can be identified by the traversed vertices $\set{v_1, \dots, v_k}$, or simply the two endpoints $v_1$ and $v_k$. We say that $v_p \in V$ is a \emph{parent} of $v_c \in V$, and $v_c$ is a \emph{child} of $v_p$, if there is a path $\set{v_R, \dots, v_p, v_c}$. We write $\child: V \rightarrow \setfunc(V)$ so that $\child(v)$ obtains the set of children of $v$. if $v_c$ We say that $v_a \in V$ is an \emph{ancestor} of $v_d \in V$, and $v_d$ is a \emph{descendant} of $v_a$, if there is a path $\set{v_R, \dots, v_a, \dots, v_d}$. We write $\anc: V \rightarrow \setfunc(V)$ so that $\anc(v)$ obtains the set of ancestors of $v$ ($\desc: V \rightarrow \setfunc(V)$ for descendants).

\subsection{Distributive Normal Forms: Intuition}
\label{subsec:prelim:mot}

The role of a dnf of a first-order formula is analogous to that of a disjunctive normal form of a propositional formula in that the dnf of a formula is a disjunction of mutually exclusive possibilities. That we can exhaustively describe mutually exclusive possibilities in the first-order setting is not obvious as the domain of quantification can be infinite and individuals in the domain can become related to one another as more individuals are considered. We start with an example to illustrate the basic problem and solution due to Hintikka.

Consider a first-order theory with one binary predicate $<$, where $x < y \eqdef <(x, y)$ is infix for ``x is less than y", for describing individuals and their order relations with one another. We can look at what the normal form of the statement ``every individual has an individual that is smaller than it", encoded in this language as
\[
(\forall x) (\exists m) m < x
\]
could be. Assuming that we have a constant that names each element in the domain of quantification, a first attempt would be to translate each $\forall$ into a conjunction (over the domain of individuals) and each $\exists$ into a disjunction (over the domain of individuals), and use a propositional normal form. That is, we convert the result of translating the quantifiers away
\[
\lAnd_{x} \left(\lOr_{m} m < x \right) \,
\]
into disjunctive normal form. Unfortunately, the domain of quantification can be infinite, so the resulting formula may be of infinite size. The ``trick" for circumventing this is to enumerate how the predicates at hand can describe the relationships between $k$ individuals (uniformly in $k$) instead of enumerating tuples of individuals. We can then identify possible \emph{kinds} of worlds by listing which kinds of individuals exist or not.

To see how this works, we rewrite the original statement as
\[
\lnot (\exists x) \lnot ((\exists m) (m < x)) \,.
\]
(In words, it is impossible to find an individual that does not have an individual that is less than it.) In this form, we can think of the normal form of a statement with quantification as describing whether kinds of individuals with certain relations to one another exist or not. In order to exhaust all the possibilities, we need to consider all the cases in which $x$ and $m$ can related to one another that are consistent with the original formula.

We can see this better in our specific example by introducing notation that enumerates all descriptions of one and two free individual variables describable by the predicate $M$. When there is one free individual variable $x_1$, the only possibility is to relate $x_1$ to itself as below
\[
P_{a_1}(x_1) \eqdef (\pm)^{a_1}(x_1 < x_1)
\]
where $a_1 = 0$ says that $x_1$ is not less than itself and $a_1 = 1$ says that $x_1$ is less than itself. When there are two free individual variables $x_1$ and $x_2$, we have
\[
Q_{a_1a_2a_3}(x_1, x_2) \eqdef (\pm)^{a_1}(x_1 < x_2) \land (\pm)^{a_2}(x_2 < x_1) \land (\pm)^{a_3}(x_2 < x_2)
\]
where the subscript $a_1a_2a_3a_4$ indexes each $Q$. For example,
\[
Q_{100}(x_1, x_2) = x_1 < x_2 \land \lnot(x_2 < x_1) \land \lnot(x_2 < x_2) \,.
\]
We enumerate all combinations of whether such individuals exist or not next.
\begin{multline*}
\mvconst_{b_1\dots b_{2^{512}}} \eqdef \\
(\pm)^{b_1} [(\exists x_1) P_0(x_1) \land \lnot (\exists x_2) Q_{000}(x_1, x_2) \land \dots \land \lnot (\exists x_2) Q_{111}(x_1, x_2)] \land \\
\ldots \land \\
(\pm)^{b_{2^{256}}} [(\exists x_1) P_0(x_1) \land (\exists x_2) Q_{000}(x_1, x_2) \land \dots \land (\exists x_2) Q_{111}(x_1, x_2)] \land \\
(\pm)^{b_{2^{256}+1}} [(\exists x_1) P_1(x_1) \land \lnot (\exists x_2) Q_{000}(x_1, x_2) \land \dots \land \lnot (\exists x_2) Q_{111}(x_1, x_2)] \land \\
\ldots \land \\
(\pm)^{b_{2^{512}}} [(\exists x_1) P_1(x_1) \land (\exists x_2) Q_{000}(x_1, x_2) \land \dots \land (\exists x_2) Q_{111}(x_1, x_2)]
\end{multline*}
The possible kinds of worlds described by our original formula is then any
\[
\mvconst_{b_1\dots b_{2^{512}}}
\]
that implies the original formula. When we introduce dnfs more formally (Section~\ref{subsec:prelim:dnf}), we will see that the possible kinds of worlds are \emph{constituents}.

The example deserves some remarks. First, note that we really have exhaustively enumerated all the mutually exclusive possibilities. The possibility $\mvconst_{0\dots0}$ describes one extreme where there are no individuals (and hence the original statement is vacuously true),\footnote{Note that traditional presentations of first-order model theory disallow empty domains although this restriction is not necessary. On the syntactic side, we will need to modify proof rules (\eg, the rule $(\forall x) \rightarrow (\exists x)$ used in converting formula to prenex normal form no longer holds) to maintain soundness and completeness.} the possibility $\mvconst_{0\dots0b^{2^{256}+1}\dotsb^{2^{512}}}$ requires individuals to be less than themselves, and the possibility $\mvconst_{1\dots1}$ enables every kind of individual to exist with respect to $<$. Second, note that the number of possibilities even in this small example (two individuals and one predicate) is enormous at $2^{512}$. The astronomical number of constituents is not an issue for theoretical purposes although it does render the straightforward application of the theory to be unfeasible.

\subsection{Distributive Normal Forms: Background}
\label{subsec:prelim:dnf}

Define the set
\[
\mathbf{S}(\set{\mvform_1, \dots, \mvform_k}) \eqdef \set{ \lAnd_{i \in \set{1, \dots, k}} (\pm)^{b_i} \mvform_i \ST b_1 \in \two, \dots, b_k \in \two} \,.
\]
An element of $\mathbf{S}(\set{\mvform_1, \dots, \mvform_k})$ is a conjunction of every $\mvform_i$ or its negation.

Let $\cA[y_1, \dots, y_k]$ denote the set of all atomic formula (\ie, a predicate applied to a tuple of terms) involving the free individual terms (\ie, constants or variables) $y_1, \dots, y_k$. Let $\cB[y_1, \dots, y_k]$ denote the subset of $\cA[y_1, \dots, y_k]$ that mentions $y_k$ at least once.

\paragraph{Attributive constituents}
An \emph{attributive constituent} with $k$ free individual terms $y_1, \dots, y_k$ of depth $0$ is an element of $\mathbf{S}(\cB[y_1, \dots, y_k])$. We write $\aconsts^{(0)}[y_1, \dots, y_k] \eqdef \mathbf{S}(\cB[y_1, \dots, y_k])$ for the set of all attributive constituents with $k$ free individual terms $y_1, \dots, y_k$ of depth $0$. By convention, we set $\aconsts^{(0)}[] = \set{\true}$. An attributive constituent of depth $0$ is a formula of the form
\[
\mvaconst^{(0)}_r[y_1, \dots, y_k] = \lAnd_{i \in \set{1, \dots, \ell^\cB_k}} (\pm)^{b_i} B_i[y_1, \dots, y_k]
\]
where $\ell^\cB_k \eqdef |\cB[y_1, \dots, y_k]|$, each $b_i \in \two$, and each $B_i[y_1, \dots, y_k] \in \cB[y_1, \dots, y_k]$. The subscript $r$ indexes the attributive constituent and can be identified with the string $b_1 \dots b_{\ell^\cB_k}$. Let $\cG^0_k \eqdef \strfunc^{\ell^\cB_k}(\two)$ be an index set for attributive constituents with $k$ free individual terms of depth $0$. We have $\cG^0_k \cong \aconsts^{(0)}[y_1, \dots, y_k]$. The superscript $(0)$ indicates the \emph{depth} of the formula, \ie, the maximal number of nested quantifiers in the formula. Hence a depth of $0$ indicates that there are no quantifiers.

The set of attributive constituents $\aconsts^{(d)}[y_1, \dots, y_k]$ of depth $d > 0$ is defined by induction on $d$. More concretely, we have an \emph{attributive constituent} with $k$ free individual terms $y_1, \dots, y_k$ of depth $d > 0$ has the form
\begin{multline*}
\mvaconst^{(d)}_{r, s}[y_1, \dots, y_k] = \mvaconst^{(0)}_r[y_1, \dots, y_k] \\ \land
\begin{cases}
\lAnd_{r' \in \cG^0_{k+1}} (\pm)^{s(r')} (\exists x) \mvaconst^{(0)}_{r'}[y_1, \dots, y_k, x] & \mbox{$d = 1$}\\
\lAnd_{(r', s') \in \cG^{d-1}_{k+1}} (\pm)^{s(r', s')} (\exists x) \mvaconst^{(d-1)}_{r', s'}[y_1, \dots, y_k, x] & \mbox{$d > 1$}
\end{cases}
\end{multline*}
where we will explain the undefined notation below. Let $\cG^{d}_{k} \eqdef \cG^0_k \times (\cG^{d-1}_{k + 1} \rightarrow \bm{2}) \cong \aconsts^{(d)}[y_1, \dots, y_k]$ be an index set for attributive constituents of depth $d > 0$ with $k$ free individual terms $y_1, \dots, y_k$. The subscript $(r, s) \in \cG^d_k$ is a pair of $r \in \cG^0_k$ and a function $s: \cG^{d-1}_{k+1} \rightarrow \two$ indicating whether the appropriately indexed attributive constituent (of depth $d-1$ with $k+1$ free individual terms) exists or not. When the indices do not matter, we will abbreviate $\mvconst^{(d)}_{r, s}[y_1, \dots, y_k]$ as $\mvconst^{(d)}[y_1, \dots, y_k]$. When we refer to two distinct attributive constituents whose indices do not matter, we will overload the subscripts as in $\mvconst^{(d)}_i$ and $\mvconst^{(d)}_j$ to distinguish them.

An attributive constituent with $k$ free individual terms $y_1, \dots, y_k$ of depth $d \geq 0$ can equivalently be defined as
\begin{multline*}
\mvaconst^{(d)}_{r, s}[y_1, \dots, y_k] = \mvaconst^{(0)}_r[y_1, \dots, y_k] \land \lAnd_{(r', s') \in \cG^{d-1}_{k+1}|^s_+} (\exists x) \mvaconst^{(d-1)}_{r', s'}[y_1, \dots, y_k, x] \\
\land \left( (\forall x) \lOr_{(r', s') \in \cG^{d-1}_{k+1}|^s_+} \mvaconst^{(d-1)}_{r', s'}[y_1, \dots, y_k, x] \right)
\end{multline*}
where $\cG^d_k|^s_+ \eqdef \set{ (r', s') \ST s(r', s') = 1 }$ is the index set restricted to the positive ones as given by the function $s$.

\paragraph{Constituents}
A \emph{constituent} with $k$ free individual terms $y_1, \dots, y_k$ of depth $d \geq 0$ is a formula of the form
\[
\mvconst^{(d)}_{q, r, s}[y_1, \dots, y_k] = A_q[y_1, \dots, y_{k-1}] \land \mvaconst^{(d)}_{r, s}[y_1, \dots, y_k]
\]
where $A_q \in \mathbf{S}(\cA[y_1, \dots, y_k])$. Let $\consts^{(d)}[y_1, \dots, y_k]$ be the set of constituents of depth $d$ with $k$ free individual terms. By convention, we set $\consts^{(0)}[] = \set{\true}$. We write $\cD^d_k \cong \consts^{(d)}[y_1, \dots, y_k]$ for the set indexing $\consts^{(d)}[y_1, \dots, y_k]$. We use the same abbreviation scheme for the indices of constituents as we did for attributive constituents. Note that a constituent is an attributive constituent with an additional $A_q[y_1, \dots, y_{k-1}]$. Thus attributive constituents and constituents can be identified when there are $0$ free individual terms.

\paragraph{Distributive normal forms}
A \emph{distributive normal form} (dnf) with $k$ free individual terms $y_1, \dots, y_k$ is a disjunction of constituents
\[
\lOr_{\mvconst^{(d)}[y_1, \dots, y_k] \in D} \mvconst^{(d)}[y_1, \dots, y_k]
\]
for some subset $D \subseteq \consts^{(d)}[y_1, \dots, y_k]$ of constituents.

\paragraph{Properties}
Attributive constituents, constituents, and dnfs have the following useful properties~\citep{hintikka1965distributive}.
\begin{proposition}[Existence, mutual exclusion, and exclusivity]\hfill
\begin{description}[noitemsep]
    \item[Existence] Every formula $\mvform[y_1, \dots, y_k]$ (of depth $d$) has a distributive normal form (of depth $d$), \ie, there is a function $\dnf: \cL[y_1, \dots, y_k] \rightarrow \setfunc(\consts^{(d)}[y_1, \dots, y_k])$ such that
    \[
    \mvform^{(d)}[y_1, \dots, y_k] = \lOr_{\mvconst^{(d)}[y_1, \dots, y_k] \in \dnf(\mvform[y_1, \dots, y_k])} \mvconst^{(d)}[y_1, \dots, y_k]
    \]
    where $\cL[y_1, \dots, y_k]$ is the set of well-formed first-order sentences with free individual terms $y_1, \dots, y_k$.
    \item[Mutual exclusion] Any two constituents and attributive constituents of the same depth are mutually exclusive, \ie, $\mvconst^{(d)}_i \implies \lnot \mvconst^{(d)}_j$ for any $\mvconst^{(d)}_i \neq \mvconst^{(d)}_j$.
    \item[Expansion] Every constituent $\mvconst^{(d)}[y_1, \dots, y_k]$ can be written as a disjunction of its expansion constituents, \ie, there is a function $\expand: \N \times \consts^{(d)} \rightarrow \setfunc(\consts^{(d + e)})$ such that
    \[
    \mvconst^{(d)}[y_1, \dots, y_k] \equiv \lOr_{\mvconst^{(d+e)}[y_1, \dots, y_k] \in \expand(e, \mvconst^{(d)}[y_1, \dots, y_k])} \mvconst^{(d+e)}[y_1, \dots, y_k] \,.
    \]
\end{description}
\label{prop:prelim:prop}
\end{proposition}

Any $\mvconst^{(d+e)}[y_1, \dots, y_k] \in \expand(e, \mvconst^{(d)}[y_1, \dots, y_k])$ is said to \emph{refine} or is a \emph{refinement} of $\mvconst^{(d)}[y_1, \dots, y_k]$.\footnote{The original terminology that Hintikka uses is \emph{subordinate}. We prefer the term refinement because it evokes the intuition that $\mvconst^{(d+e)}[y_1, \dots, y_k] \in \expand(e, \mvconst^{(d)}[y_1, \dots, y_k])$ describes the possibility described by $\mvconst^{(d)}[y_1, \dots, y_k]$ in finer detail.} We write $\mvconst^{(d)}[y_1, \dots, y_k] \expandsinto \mvconst^{(d+e)}[y_1, \dots, y_k]$ when $\mvconst^{(d+e)}[y_1, \dots, y_k]$ refines $\mvconst^{(d)}[y_1, \dots, y_k]$. Let
\[
\consts[y_1, \dots, y_k] \eqdef \bigcup_{d \in \N} \consts^{(d)}[y_1, \dots, y_k] \,.
\]
Then the refinement relation $\expandsinto: \consts[y_1, \dots, y_k] \times \consts[y_1, \dots, y_k] \rightarrow \two$ is a partial order and $(\consts[y_1, \dots, y_k], \expansion)$ is a poset.

It is well-known that validity of first-order formulas is undecidable. Consequently, the consistency of constituents in a dnf is undecidable. There is a weaker notion called \emph{trivial inconsistency} that is decidable. There are several notions of trivial inconsistency~\citep[\eg, see][]{hintikka1973logic,nelte1997formulas}, although the exact form is not important for our purposes.
\begin{proposition}[Completeness]
An attributive constituent is inconsistent iff all of its expansions at some depth are trivially inconsistent~\citep{hintikka1965distributive}.
\end{proposition}
\noindent Thus, an inconsistency at depth $d$ will eventually manifest itself as trivially inconsistent at some depth $e \geq d$, although the depth $e$ is not recursively computable.\footnote{There are notions of trivial inconsistency that are not strong enough to ensure completeness as noted by~\citet{nelte1997formulas}.} The main idea is show that a consistent attributive constituent always has an expansion that is not trivially inconsistent; the result follows from an application of K\H{o}nig's tree lemma.

\section{Representing Beliefs in Mathematical Knowledge}
\label{sec:repr}

In this section, we introduce a representation that assigns probabilities to the exhaustive and mutually exclusive possibilities of first-order logic that we have just seen to be constituents. More concretely, we formalize a method for assigning weights to constituents and an appropriate Bayesian update following the idea of assigning weights to constituents described by~\citet[pg. 274--282]{hintikka1970surface} (Section~\ref{subsec:repr:ht}). The representation induces a probability distribution on the validity of first-order statements that does not enforce that logically equivalent statements are assigned the same probability so that the beliefs of agents that are not logically omniscient\footnote{The problem of logical omniscience is an issue encountered in epistemic logic~\citep[\eg, see][]{sim1997epistemic,fagin2004reasoning,halpern2011dealing} where we reason about the knowledge of agents. One solution for weakening logical omniscience involves modeling \emph{impossible possible worlds}, \ie, modeling worlds that an agent considers possible but are eventually revealed to be impossible. Hintikka argues that dnfs provide a model of impossible possible worlds---an impossible possible world is an inconsistent constituent that is not trivially inconsistent at some depth and revealed to be trivially inconsistent at a later depth~\citep[by completeness,][]{hintikka1979impossible}. Thus the application of dnfs to address the problem of logical omniscience has also been hinted at by Hintikka.} can be encoded (Section~\ref{subsec:repr:prob}). At the end of the section, we identify an embedding space---a Hilbert space---for first-order statements based on the probabilistic representation where mutual exclusion in logic translates into orthogonality in the space (Section~\ref{subsec:repr:embed}).

\begin{remark}[Simple first-order languages]
For simplicity, we restrict attention to first-order languages with a finite number of predicates, no function symbols, and no constants unless stated otherwise.\footnote{As a reminder, the effect of equality is to give an exclusive interpretation of quantifiers. All the results that hold on constituents in first-order logic without equality also hold on constituents in first-order logic with equality with the appropriate modifications.  

Note that functions can be encoded as predicates in first-order logic with equality. 

Observe also that the current setting actually admits a finite number of constants. More concretely, we can associate each constant $c$ with a monadic predicate $\mvpred_c$ where the interpretation of $\mvpred_c(x)$ is ``$x$ is the constant $c$". Any formula $\mvform[c]$ that refers to the constant $c$ can thus be translated to $(\exists x) \mvpred_c(x) \land \mvform[x]$ where $x$ is not free in $\mvform$ and we add the additional axiom $(\exists x) \mvpred_c(x)$ to the theory. Hence, we are roughly working with first-order languages with a finite number of predicates, functions, and constants.} The constant-free restriction simplifies the form of constituents and dnfs we will need to consider. As a reminder, every first-order formula $\mvform[y_1, \dots, y_k]$ with $k$ free individual terms $y_1, \dots, y_k$ has a dnf of depth $d$ constituents. In a constant-free setting, the free individual terms $y_1, \dots, y_k$ are all variables. Thus the original formula is equivalent to its universal closure $(\forall y_1) \dots (\forall y_k) \mvform[y_1, \dots, y_k]$, which is a formula with $0$ free individual terms (\ie, a sentence). Consequently, we only need to consider the set of constituents $\consts^{(d)}_0[]$, abbreviated $\consts^{(d)}$, of depth $d$ with $0$ free individual terms. We have that $\consts^{(0)} \eqdef \set{\true}$ by convention.
\end{remark}

\subsection{Hintikka Trees}
\label{subsec:repr:ht}

We formalize a representation that assigns probabilities to constituents in this section. As the set of constituents at any depth exhaust and describe all mutually exclusive possibilities at that depth, the idea behind the representation is the standard one: list all possibilities and assign weights to them that sum to one. We construct the representation in two parts. First, we introduce a \emph{refinement tree} that keeps track of the refinement relation because constituents of different depths do not denote mutually exclusive possibilities when they are related according to the refinement partial order (Section~\ref{subsubsec:repr:ht:refine}). Second, we describe how to assign weights to the refinement tree which completes the static representation of an agent's beliefs (Section~\ref{subsubsec:repr:ht:weights}). After we introduce the representation, we introduce dynamics via a renormalization operator which can be interpreted as a form of Bayesian update for beliefs (Section~\ref{subsubsec:repr:ht:renorm}).

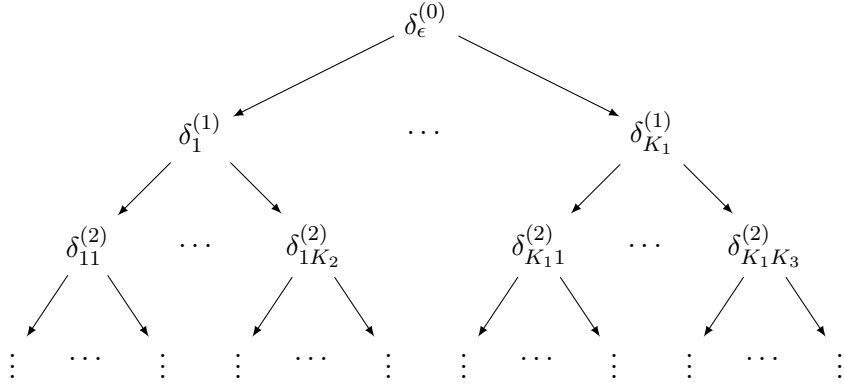
\begin{figure}[t]
    \centering
    \begin{tikzpicture}[edge from parent/.style={draw,-latex},level/.style={sibling distance=60mm/#1}]
\node [] (z){$\mvconst^{(0)}_\epsilon$}
child {
  node (a) {$\mvconst^{(1)}_{1}$}
  child {
    node (b) {$\mvconst^{(2)}_{11}$}
      child { node (c) {$\vdots$} } 
      child { node (d) {$\vdots$} }
  }
  child {
    node (g) {$\mvconst^{(2)}_{1K_2}$}
      child { node (e) {$\vdots$} }
      child { node (f) {$\vdots$} }
  }
}
child {
  node (j) {$\mvconst^{(1)}_{K_1}$}
  child {
    node (k) {$\mvconst^{(2)}_{K_1 1}$}
      child { node (m) {$\vdots$} } 
      child { node (n) {$\vdots$} }
  }
  child {
    node (l) {$\mvconst^{(2)}_{K_1 K_3}$}
      child { node (o) {$\vdots$} }
      child { node (p) {$\vdots$} }
  }
}
;
\path (a) -- (j) node [midway] {\dots};
\path (b) -- (g) node [midway] {\dots};
\path (k) -- (l) node [midway] {\dots};
\path (c) -- (d) node [midway] {\dots};
\path (e) -- (f) node [midway] {\dots};
\path (m) -- (n) node [midway] {\dots};
\path (o) -- (p) node [midway] {\dots};
\end{tikzpicture}
    \caption{A depiction of a refinement tree $(\consts, E)$. Each vertex represents a constituent and each edge indicates a refinement relation. For example, the constituent $\mvconst^{(2)}_{11}$ occurs in the expansion of $\mvconst^{(1)}_{1}$. We assume that each constituent's set of refinements is an indexed set so that $\mvconst^{(2)}_{11}$ indicates that we take the first refinement of $\mvconst^{(0)}_\epsilon$ and then take the first refinement of $\mvconst^{(1)}_1$.}
    \label{fig:repr:et}
\end{figure}

\subsubsection{Refinement tree}
\label{subsubsec:repr:ht:refine}

Let the set of vertices be the set of constituents of any depth $\consts$. Let the set of edges $\xi \eqdef \set{ (\mvconst^{(d)}, \mvconst^{(d+1)}) \ST \mvconst^{(d)} \expandsinto \mvconst^{(d+1)}}$ consist of the refinement relation omitting reflexive relations. Then $(\consts, \xi)$ is a graph that encodes the refinement relation (minus the reflexive edges).

The graph $(\consts, \xi)$ is not a tree because the expansions of two distinct constituents can share refining constituents, although the shared constituents are necessarily inconsistent.
\begin{proposition}
Suppose $\mvconst^{(d)}_i \neq \mvconst^{(d)}_j$. If $\mvconst^{(d+e)} \in \expand(e, \mvconst^{(d)}_i) \cap \expand(e, \mvconst^{(d)}_j)$, then $\mvconst^{(d+e)}$ is inconsistent.
\end{proposition}
\begin{proof}
Assume additionally for the sake of contradiction that $\mvconst^{(d+e)}$ is consistent. Then there exists a structure $\cM$ such that $\cM \vDash \mvconst^{(d+e)}$. Thus we have that $\cM \vDash \mvconst^{(d)}_i$ and $\cM \vDash \mvconst^{(d)}_j$ (because $\mvconst^{(d+e)} \in \expand(e, \mvconst^{(d)}_i) \cap \expand(e, \mvconst^{(d)}_j)$ by another assumption), which contradicts that $\mvconst^{(d)}_i$ and $\mvconst^{(d)}_j$ are mutually incompatible (by exclusivity in Proposition~\ref{prop:prelim:prop}).
\end{proof}

By the proposition above, we can associate any shared constituent that is a refinement of two parent constituents to either parent constituent and disassociate it with the other without changing the consistency of either parent constituent. In other words, we can remove one edge. We can use this observation to convert $(\consts, \xi)$ into a tree. We call $(\consts, \xi_R)$ a \emph{refinement tree} where $\xi_R$ is the set of edges obtained after the pruning procedure described above is applied. Throughout the rest of this paper, we will assume that we have chosen one such pruning and will write $(\consts, \xi_R)$ as $(\consts, \xi)$. We will also overload $\expandsinto$ to refer to the pruned refinement partial order. We have the following obvious relationships between the (pruned) refinement partial order and (pruned) refinement tree.
\begin{proposition}\hfill
\begin{enumerate}[noitemsep]
    \item $\mvconst^{(d)} \expandsinto \mvconst^{(d+1)}$ iff $\mvconst^{(d+1)} \in \child(\mvconst^{(d)})$.
    \item $\mvconst^{(d)} \expandsinto \mvconst^{(e)}$ iff $\mvconst^{(d)} = \mvconst^{(e)}$ or $\mvconst^{(d)} \in \anc(\mvconst^{(e)})$ (equivalently $\mvconst^{(e)} \in \desc(\mvconst^{(d)})$).
\end{enumerate}
\end{proposition}
\begin{proof}
Straightforward.
\end{proof}

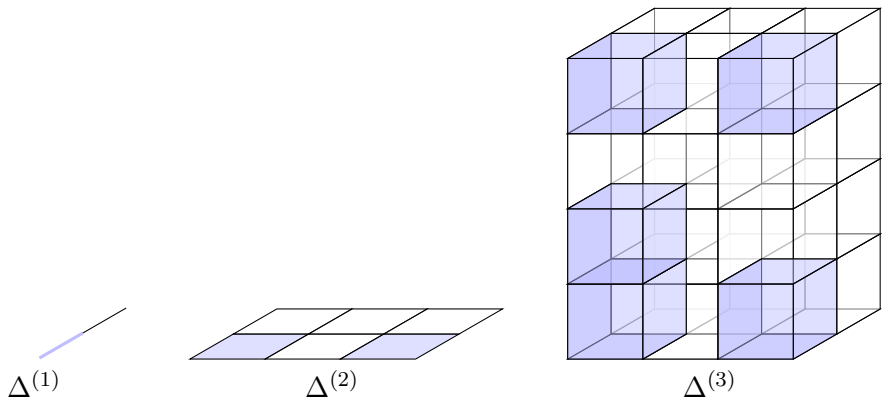
\begin{figure}[t]
    \centering
    \newcommand{\drawline}[4]{
    \pgfmathsetmacro \angle {30}
    \pgfmathsetmacro \xd {{2/3*cos(\angle)}}
    \pgfmathsetmacro \yd {{2/3*sin(\angle)}}
    \pgfmathsetmacro \x {{#1-1+(#2-1)*(\xd)}}
    \pgfmathsetmacro \y {{#3-1+(#2-1)*(\yd)}}

    \draw[#4] (\x,\y+1) -- (\x+\xd,\y+1+\yd);
}

\newcommand{\drawhor}[3]{
  \draw[#3] (\x,\y) -- (\x+1,\y);
}

\newcommand{\drawvert}[3]{
  \draw[#3] (\x,\y) -- (\x,\y+1);
}

\newcommand{\drawlinea}[4]{
    \pgfmathsetmacro \angle {30}
    \pgfmathsetmacro \xd {{2/3*cos(\angle)}}
    \pgfmathsetmacro \yd {{2/3*sin(\angle)}}
    \pgfmathsetmacro \x {{#1-1+(#2-1)*(\xd)}}
    \pgfmathsetmacro \y {{#3-1+(#2-1)*(\yd)}}

    \draw[#4, very thick] (\x,\y+1) -- (\x+\xd,\y+1+\yd);
}

\newcommand{\drawslant}[4]{
    \pgfmathsetmacro \angle {30}
    \pgfmathsetmacro \xd {{2/3*cos(\angle)}}
    \pgfmathsetmacro \yd {{2/3*sin(\angle)}}
    \pgfmathsetmacro \x {{#1-1+(#2-1)*(\xd)}}
    \pgfmathsetmacro \y {{#3-1+(#2-1)*(\yd)}}

    \draw[fill=#4, fill opacity=0.5] (\x,\y+1) -- (\x+\xd,\y+1+\yd) -- (\x+1+\xd,\y+1+\yd) -- (\x+1,\y+1) -- cycle;
}

\newcommand{\drawbox}[4]{
    \pgfmathsetmacro \angle {30}
    \pgfmathsetmacro \xd {{2/3*cos(\angle)}}
    \pgfmathsetmacro \yd {{2/3*sin(\angle)}}
    \pgfmathsetmacro \x {{#1-1+(#2-1)*(\xd)}}
    \pgfmathsetmacro \y {{#3-1+(#2-1)*(\yd)}}

    \draw[fill=#4, fill opacity=0.5] (\x,\y) -- (\x+\xd,\y+\yd) -- (\x+1+\xd,\y+\yd) -- (\x+1,\y) -- cycle;

    \draw[fill=#4, fill opacity=0.5] (\x,\y+1) -- (\x+\xd,\y+1+\yd) -- (\x+\xd,\y+\yd) -- (\x,\y) -- cycle;

    \draw[fill=#4, fill opacity=0.5] (\x,\y) -- (\x+1,\y) -- (\x+1,\y+1) -- (\x,\y+1) -- cycle;
    
    \draw[fill=#4, fill opacity=0.5] (\x,\y+1) -- (\x+\xd,\y+1+\yd) -- (\x+1+\xd,\y+1+\yd) -- (\x+1,\y+1) -- cycle;
    
    \draw[fill=#4, fill opacity=0.5] (\x+1,\y+1) -- (\x+1+\xd,\y+1+\yd) -- (\x+1+\xd,\y+\yd) -- (\x+1,\y) -- cycle;
}

\begin{tabular}{p{2cm}p{4.6cm}p{4.6cm}}
    \centering
    {\begin{tikzpicture}
    \drawlinea{0}{0}{0}{blue!25}
    \drawline{0}{1}{0}{black}
    
    \end{tikzpicture}} &
    {\begin{tikzpicture}
\drawslant{0}{0}{0}{blue!25}
\drawslant{0}{1}{0}{white}

\drawslant{1}{0}{0}{white}
\drawslant{1}{1}{0}{white}

\drawslant{2}{0}{0}{blue!25}
\drawslant{2}{1}{0}{white}
\end{tikzpicture}} &
    {\begin{tikzpicture}
\drawbox{0}{1}{0}{white}
\drawbox{1}{1}{0}{white}
\drawbox{2}{1}{0}{white}

\drawbox{0}{1}{1}{white}
\drawbox{1}{1}{1}{white}
\drawbox{2}{1}{1}{white}

\drawbox{0}{1}{2}{white}
\drawbox{1}{1}{2}{white}
\drawbox{2}{1}{2}{white}

\drawbox{0}{1}{3}{white}
\drawbox{1}{1}{3}{white}
\drawbox{2}{1}{3}{white}

\drawbox{1}{0}{0}{white}

\drawbox{1}{0}{1}{white}
\drawbox{2}{0}{1}{white}

\drawbox{0}{0}{2}{white}
\drawbox{1}{0}{2}{white}
\drawbox{2}{0}{2}{white}

\drawbox{1}{0}{3}{white}

\drawbox{0}{0}{0}{blue!25}
\drawbox{2}{0}{0}{blue!25}

\drawbox{0}{0}{1}{blue!25}

\drawbox{0}{0}{3}{blue!25}
\drawbox{2}{0}{3}{blue!25}
\end{tikzpicture}} \\
$\consts^{(1)}$ & \quad\quad\quad\quad $\consts^{(2)}$ & \quad\quad\quad\quad $\consts^{(3)}$
\end{tabular}
    \caption{An illustration of the set of depth $d = 1$, $2$, or $3$ constituents (\ie, the universe of possibilities by depth) where each cell corresponds to a constituent, the dimension of the cell corresponds to the depth of the constituents, and the refinement relation is encoded as projection. Cells colored light blue indicate that the associated constituent is consistent and white cells indicate that the associated constituent is inconsistent.}
    \label{fig:repr:embed:hd}
\end{figure}

A path between constituents $\mvconst^{(d)}$ and $\mvconst^{(d + e)}$ is the sequence of constituents and their refinements $\mvconst^{(d)} \expandsinto \mvconst^{(d+1)} \expandsinto \dots \expandsinto \mvconst^{(d+e)}$. Because there is only one path between any two vertices in a tree, we can identify a constituent (\ie, a node in a refinement tree) with the path taken through a refinement tree starting at the root node $\mvconst^{(0)}$ to reach it.

Figure~\ref{fig:repr:et} gives an illustration of a refinement tree where constituents are indexed by their paths. The root constituent $\mvconst^{(0)}_\epsilon$ of the tree is indexed by the empty path $\epsilon$. Figure~\ref{fig:repr:embed:hd} gives another illustration of a refinement tree.

\subsubsection{Assigning weights}
\label{subsubsec:repr:ht:weights}

We assign weights to constituents by attaching a weight to each node of the refinement tree. Because the assignment of weights needs to respect the refinement partial order, we will need a notion of coherence between the weight assignments to adjacent levels of the refinement tree.
\begin{definition}
A \emph{Hintikka tree} (HT) is a tuple $(\consts, \xi, \mvhintree)$ where $(\consts, \xi)$ is a refinement tree and $\mvhintree: \consts \rightarrow [0, 1]$ is a function on constituents satisfying
\begin{description}[noitemsep]
    \item[Unitial initial beliefs] $\mvhintree(\mvconst^{(0)}) = 1$; and
    \item[Coherently constructed] $\mvhintree(\mvconst^{(d)}) = \sum_{\mvconst^{(d+1)} \expansion \mvconst^{(d)}} \mvhintree(\mvconst^{(d+1)})$.\footnote{The method of assigning weights in a HT is slightly different than the one described in prose by~\citet[pg. 274--282]{hintikka1970surface}. In particular, Hintikka combines the statics and dynamics of the weight assignment whereas we separate them out and only describe the statics here. We will discuss the dynamics in Section~\ref{subsubsec:repr:ht:renorm}.}
\end{description}
We will abbreviate a HT $(\consts, \xi, \mvhintree)$ as $\mvhintree$. We write $\mathbf{HT}(\cL)$ for the set of HTs defined with respect to the the first-order simple language $\cL$.
\end{definition}
\noindent The first condition states that we start off with unitial beliefs. The second condition enforces that the probability that we assign a constituent $\mvconst^{(d)}$ is contained entirely within the subtree of the refinement tree rooted at $\mvconst^{(d)}$. Hence the assignment of weights is conserved across depth. Observe that the assignment of weights to constituents is not constrained by the ``fraction" of models that the constituents are satisfiable in. If it were, then the induced distribution on the validity of first-order statements would enforce logical omniscience.

\begin{proposition}[Normalization]
The beliefs assigned to constituents at each depth $d \in \N$ by a HT $\mvhintree$ are normalized:
\[
\sum_{\mvconst^{(d)} \in \consts^{(d)}} \mvhintree^d(\mvconst^{(d)}) = 1 \,.
\]
\label{prop:repr:norm}
\end{proposition}
\begin{proof}
We proceed by induction on $d$. The base case follows from unitial initial beliefs. In the inductive case, we have to show that
\[
\sum_{\mvconst^{(d+1)} \in \consts^{(d+1)}} \mvhintree(\mvconst^{(d+1)}) = 1 \,.
\]
We have that
\begin{align*}
\sum_{\mvconst^{(d+1)} \in \consts^{(d+1)}} \mvhintree(\mvconst^{(d+1)}) & = \sum_{\mvconst^{(d)} \in \consts^{(d)}} \sum_{\mvconst^{(d+1)} \expansion \mvconst^{(d)}} \mvhintree(\mvconst^{(d+1)}) \\
& = \sum_{\mvconst^{(d)} \in \consts^{(d)}} \mvhintree(\mvconst^{(d)})
\end{align*}
where the first equality is a rearrangement and the second equality follows because $\mvhintree$ is coherently constructed. The result follows as we have $\sum_{\mvconst^{(d)} \in \consts^{(d)}} \mvhintree(\mvconst^{(d)}) = 1$ by the inductive hypothesis.
\end{proof}
\begin{proposition}[Infinite supported path]
For any HT $\mvhintree$, there is a chain of constituent $\mvconst^{(0)} \expandsinto \mvconst^{(1)} \expandsinto \mvconst^{(2)} \expandsinto \dots$ such that $\mvhintree(\mvconst^{(d)}) > 0$ for any $\mvconst^{(d)}$ in the chain.
\label{prop:repr:infpath}
\end{proposition}
\begin{proof}
We proceed by induction on $d$. The base case follows by unitial initial beliefs of $\mvhintree$. In the inductive case, we have that $\mvhintree(\mvconst^{(d)}) > 0$. The result follows as $\mvhintree$ is coherently constructed and $\mvconst^{(d)}$ has a finite number of children so there must exist a refinement $\mvconst^{(d+1)} \expansion \mvconst^{(d)}$ such that $\mvhintree(\mvconst^{(d+1)}) > 0$.
\end{proof}

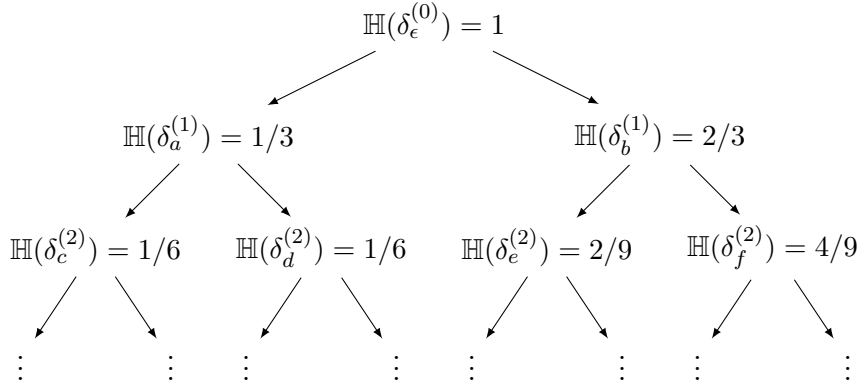
\begin{figure}[t]
    \centering
    \begin{tikzpicture}[edge from parent/.style={draw,-latex},level/.style={sibling distance=60mm/#1}]
\node [] (z){$\mvhintree(\mvconst^{(0)}_\epsilon) = 1$}
child {
  node (a) {$\mvhintree(\mvconst^{(1)}_{a}) = 1/3$}
  child {
    node (b) {$\mvhintree(\mvconst^{(2)}_{c}) = 1/6$}
      child { node (c) {$\vdots$} } 
      child { node (d) {$\vdots$} }
  }
  child {
    node (g) {$\mvhintree(\mvconst^{(2)}_{d}) = 1/6$}
      child { node (e) {$\vdots$} }
      child { node (f) {$\vdots$} }
  }
}
child {
  node (j) {$\mvhintree(\mvconst^{(1)}_{b}) = 2/3$}
  child {
    node (k) {$\mvhintree(\mvconst^{(2)}_{e}) = 2/9$}
      child { node (m) {$\vdots$} } 
      child { node (n) {$\vdots$} }
  }
  child {
    node (l) {$\mvhintree(\mvconst^{(2)}_{f}) = 4/9$}
      child { node (o) {$\vdots$} }
      child { node (p) {$\vdots$} }
  }
}
;
\end{tikzpicture}
    \caption{A drawing of an example Hintikka tree (HT). Vertices with $0$ belief are not shown. Each level of the HT is normalized. Moreover, the probability assigned each subtree is conserved.}
    \label{fig:repr:htex}
\end{figure}

We end with several examples of HTs.
\begin{example}
Figure~\ref{fig:repr:htex} gives an illustration of an example HT. As required by the definition, beliefs across depth are coherent.
\end{example}
\begin{example}
A HT is an \emph{uninformative Hintikka tree} if $\mvhintree$ is a uniform distribution at every depth, \ie, $\mvhintree(\mvconst^{(d)}) = 1/|\consts^{(d)}|$ for any $\mvconst^{(d)} \in \consts^{(d)}$.
\end{example}
\begin{example}
A HT is a \emph{depth Hintikka tree} if $\mvhintree$ is constrained so that inconsistent constituents are assigned $0$.\footnote{The terminology is inspired by depth information~\citep{hintikka1970surface}. Observe that there are consistent constituents at every depth and that consistent constituents have consistent refinements by the constituent completeness theorem so that a depth HT is well-defined. For example, the sentence $((\exists x_1) \dots (\exists x_d) \mvform[x_1, \dots, x_d]) \lor \lnot ((\exists x_1) \dots (\exists x_d) \mvform[x_1, \dots, x_d])$ is logically valid at depth $d$.\label{footnote:repr:ht:wd}} Inconsistency is undecidable so that a depth HT is not computable. If a theorem proving agent represents mathematical knowledge with a depth HT, then the agent is logically omniscient. We have that $\vDash \mvform^{(d)}$ iff $\sum_{\mvconst^{(d)} \in \dnf(\mvform^{(d)})} \mvhintree(\mvconst^{(d)}) = 1$ for some depth HT $\mvhintree$.
\end{example}

A HT provides a \emph{static} representation of an agent's beliefs. Naturally, an agent may encounter a situation where it realizes that its beliefs need to be revised. For example, upon further inspection of all the expansions of a parent constituent, the agent may realize that they are all inconsistent so the belief in the parent constituent should be eliminated and redistributed to other constituents. Intuitively, this may occur because the increase in depth corresponds to the construction of an object (\ie, an introduction of an existential) and the consideration of this extra object changes the valuation of the consistency of the parent possibility. Indeed, such a situation arises from the constituent completeness theorem: inconsistent constituents are eventually revealed to be trivially inconsistent at some depth even if they are not trivially inconsistent at shallower depths. We turn our attention to the dynamics of belief revision in the representation now.

\subsubsection{Renormalization dynamics}
\label{subsubsec:repr:ht:renorm}

\citet[pg. 281]{hintikka1970surface} describes a method of redistributing weights assigned to a refinement tree when belief in a node and all of its descendants is lost. The intuition for the update follows Bayesian ``refute" and ``rescale" dynamics: when belief in a node and all of its descendants is eliminated so that those possibilities are ``refuted", the beliefs in the smallest subtree containing that node that still has positive weight are ``rescaled" appropriately. In this section, we formalize this intuition as a \emph{renormalization} operation. Towards this end, we will specify (1) which constituents to redistribute beliefs to and (2) the amount of belief to redistribute to those constituents.

\paragraph{Part one of renormalization}
We start with the first task and begin by identifying which constituents to redistribute beliefs to when we discontinue beliefs in $\mvconst^{(d)}_-$ in a HT $\mvhintree$. Define the function $\suppo_{\mvhintree, \mvconst^{(d)}_-}: \consts \rightarrow \two$ as (1) $\suppo_{\mvhintree, \mvconst^{(d)}_-}(\mvconst^{(e)}) = \true$ if there is some $\mvconst^{(e)} \expandsinto \mvconst^{(e+1)} \expandsinto \dots \expandsinto \mvconst^{(d)}$ such that $ \mvhintree(\mvconst^{(n)}) > 0$ for $e \leq n \leq d$ and $\mvconst^{(d)} \neq \mvconst^{(d)}_-$ and (2) $\suppo_{\mvhintree, \mvconst^{(d)}_-}(\mvconst^{(e)}) = \false$ otherwise. Define the \emph{support} function $\supp: \consts \rightarrow \two$ as
\[
\supp_{\mvhintree, \mvconst^{(d)}_-}(\mvconst^{(e)}) \eqdef 
\begin{cases}
\false & \mbox{$\mvconst^{(e)} \expansion \mvconst^{(d)}_-$} \\
\suppo_{\mvhintree, \mvconst^{(d)}_-}(\mvconst^{(e)}) & \mbox{otherwise.}
\end{cases}
\]

The idea is that we will transfer beliefs assigned to unsupported constituents over to the appropriate supported constituents. Define the abbreviations $\cS^+_{\mvhintree, \mvconst^{(d)}_-} \eqdef \set{\mvconst^{(e)} \in \consts \ST \supp_{\mvhintree, \mvconst^{(d)}_-}(\mvconst^{(e)}) = \true}$ and $\cS^-_{\mvhintree, \mvconst^{(d)}_-} \eqdef \consts \backslash \cS^+_S$. Thus $\cS^-_{\mvhintree, \mvconst^{(d)}_-}$ and $\cS^+_{\mvhintree, \mvconst^{(d)}_-}$ partition $\consts$. Define a \emph{$d$-redistribution point} as
\[
\rho_{\mvhintree, \mvconst^{(d)}_-} \eqdef \max_{0 \leq r \leq d} \set{ \mvconst^{(r)} \ST \mvconst^{(r)} \expandsinto \mvconst^{(d)}, \mvconst^{(r)} \in \cS^+_{\mvhintree, \mvconst^{(d)}_-} } \,,
\]
which is the closest (\ie, deepest by depth) ancestor constituent that has supported descendants. A $d$-redistribution point identifies a vertex of the refinement tree that has supported descendants to redistribute beliefs in unsupported constituents to.

\paragraph{Part two of renormalization}
We turn our attention towards the second task concerning the amount of belief to redistribute to each constituent now. Let $D^+_{\mvhintree, \mvconst^{(d)}_-} \eqdef \child(\rho_{\mvhintree, \mvconst^{(d)}_-}) \cap \cS^+_{\mvhintree, \mvconst^{(d)}_-}$ be the children of $\rho_{\mvhintree, \mvconst^{(d)}_-}$ that are supported. Then
\[
Z^+_{\mvhintree, \mvconst^{(d)}_-} \eqdef \sum_{\mvconst^{(e)} \in D^+_{\mvhintree, \mvconst^{(d)}_-}} \mvhintree(\mvconst^{(e)})
\]
is the \emph{positive renormalization constant} and
\[
Z_{\mvhintree, \mvconst^{(d)}_-} \eqdef \sum_{\mvconst^{(e)} \in \child(\rho_{\mvhintree, \mvconst^{(d)}_-})} \mvhintree(\mvconst^{(e)})
\]
is the \emph{total renormalization constant}.

\begin{figure}[t]
    \centering
    \begin{tabular}{cc}
\begin{tikzpicture}[
edge from parent/.style={draw,-latex},
level/.style={sibling distance=60mm/#1},
level 1/.style={sibling distance=50mm},
level 2/.style={sibling distance=30mm},
level 3/.style={sibling distance=20mm}]
\node [] (z){$\mvhintree(\mvconst) = a + b + c +d$}
child {
node (b) {$\mvhintree(\mvconst_l) = a + b + c$}
  child {
    node (d) {$\mvhintree(\mvconst_{ll}) = a$} 
    child {
      node(g) {$\mvhintree(\mvconst_{lll}) = a$}
    }
  } 
  child {
    node (d) {$\mvhintree(\mvconst_{lr}) = b + c$} 
      child { node(i) {$\mvhintree(\mvconst_{lrl}) = b$} }
      child { node(j) {$\mvhintree(\mvconst_{lrr}) = c$} }
  }
}
child {
  node (c) {$\mvhintree(\mvconst_r) = d$}
  child {
    node (f) {$\mvhintree(\mvconst_{rl}) = d$}
      child { node (j) {$\mvhintree(\mvconst_{rll}) = d$} } 
  }
}
;
\end{tikzpicture}
&
\begin{tikzpicture}[edge from parent/.style={draw,-latex},
level/.style={sibling distance=30mm/#1},
level 1/.style={sibling distance=35mm},
level 2/.style={sibling distance=30mm},
level 3/.style={sibling distance=25mm}]
\node [] (z){$\mvhintree(\mvconst) = a + b + c +d$}
child {
node (b) {$\mvhintree(\mvconst_l) = a + b + c$}
  child {
    node (d) {$\mvhintree(\mvconst_{lr}) = a + b + c$} 
      child { node(i) {$\mvhintree(\mvconst_{lrl}) = \frac{ab}{b+c}$} }
      child { node(j) {$\mvhintree(\mvconst_{lrr}) = \frac{ac}{b+c}$} }
  }
}
child {
  node (c) {$\mvhintree(\mvconst_r) = d$}
  child {
    node (f) {$\mvhintree(\mvconst_{rl}) = d$}
      child { node (j) {$\mvhintree(\mvconst_{rll}) = d$} } 
  }
}
;
\end{tikzpicture}
\\
before $\renorm_{\mvconst_{lll}}(\mvhintree)$ & after $\renorm_{\mvconst_{lll}}(\mvhintree)$
\end{tabular}
    \caption{An example of renormalization. We assume that the constituent $\mvconst$ appears somewhere in the refinement tree. The left shows a $\mvhintree$ before renormalization by $\mvconst_{lll}$. The constituent $\mvconst_l$ is the $d$-renormalization point $\rho_{\mvhintree, \mvconst_{lll}}$ as it is the closest ancestor that has descendants that are supported through depth $d$. After renormalization on the right, the weight $a$ in the eliminated region (constituents $\mvconst_{ll}$ and $\mvconst_{lll}$) are transferred to the closest region that is still supported (constituents $\mvconst_{lr}$, $\mvconst_{lrl}$, and $\mvconst_{lrr}$) in proportion to the existing weights.}
    \label{fig:repr:renorm}
\end{figure}
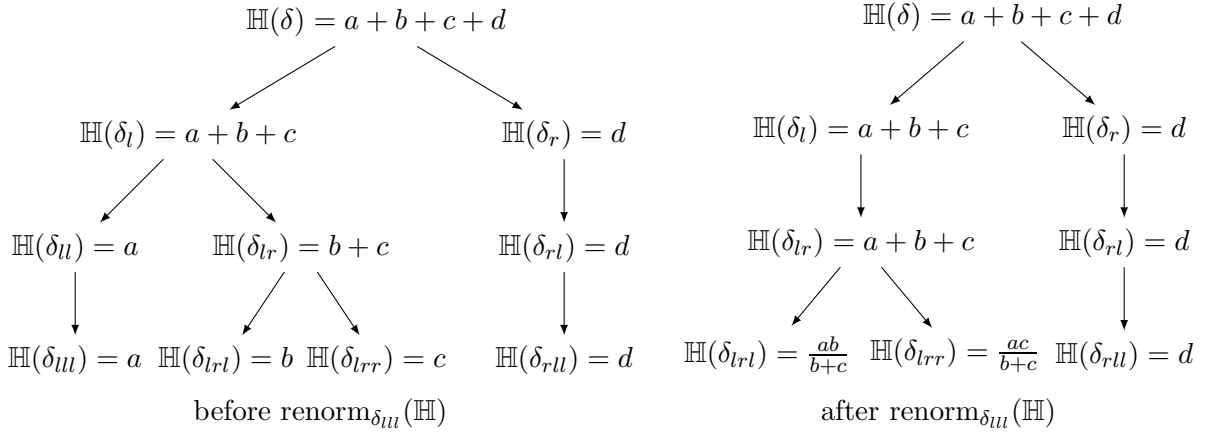

\paragraph{Renormalization}
We arrive at the definition of renormalization by putting the two parts together.
\begin{definition}
The \emph{renormalization} of $\mvhintree$ with respect to $\mvconst^{(d)}_-$ is a function $\renorm_{\mvconst^{(d)}_-}: (\consts \rightarrow \two) \rightarrow \consts \rightarrow [0, 1]$ defined as
\[
\renorm_{\mvconst^{(d)}_-}(\mvhintree)(\mvconst^{(e)}) = \begin{cases}
\frac{Z_{\mvhintree, \mvconst^{(d)}_-}}{Z^+_{\mvhintree, \mvconst^{(d)}_-}} \mvhintree(\mvconst^{(e)}) & \mbox{$\mvconst^{(e)} \in \desc(\rho_{\mvhintree, \mvconst^{(d)}_-}) \cap \cS^+_{\mvhintree, \mvconst^{(d)}_-}$} \\
0 & \mbox{$\mvconst^{(e)} \in \desc(\rho_{\mvhintree, \mvconst^{(d)}_-}) \cap \cS^-_{\mvhintree, \mvconst^{(d)}_-}$} \\
\mvhintree(\mvconst^{(e)}) & \mbox{otherwise}
\end{cases}
\]
when $Z^+_{\mvhintree, \mvconst^{(d)}_-} > 0$ and undefined otherwise.\footnote{\citet[pg. 281]{hintikka1970surface} devotes one paragraph to describing renormalization. The definition of renormalization given here translates that description into mathematical language as well as explicitly makes the connection between the redistribution of weights for unsupported constituents and Bayesian update.}
\end{definition}
\noindent Observe that $\renorm_{\mvconst^{(d)}_-}(\mvhintree)$ only affects the descendants of $\rho_{\mvhintree, \mvconst^{(d)}_-}$: $\renorm_{\mvconst^{(d)}_-}(\mvhintree)(\mvconst^{(e)}) = \mvhintree(\mvconst^{(e)})$ for $\mvconst^{(e)} \notin \desc(\rho_{\mvhintree, \mvconst^{(d)}_-})$. Figure~\ref{fig:repr:renorm} provides an illustration of the renormalization process.

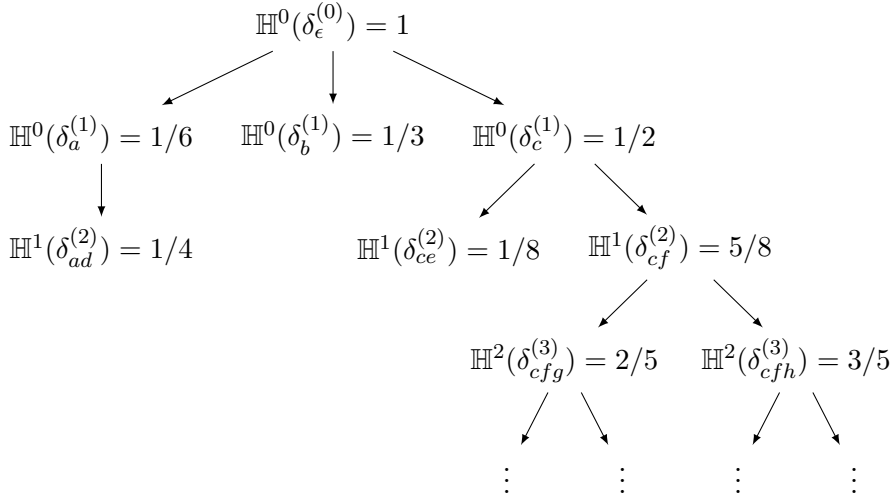
\begin{figure}[t]
    \centering
    \begin{tikzpicture}[
  edge from parent/.style={draw,-latex},
  level 1/.style={sibling distance=8em},
  level 2/.style={sibling distance=8em},
  level 3/.style={sibling distance=8em},
  level 4/.style={sibling distance=4em}]
\node [] (z){$\mvhintree^0(\mvconst^{(0)}_\epsilon) = 1$}
child {
  node (a) {$\mvhintree^0(\mvconst^{(1)}_{a}) = 1/6$}
  child {
    node (b) {$\mvhintree^1(\mvconst^{(2)}_{ad}) = 1/4$}
  }
}
child {
  node (df) {$\mvhintree^0(\mvconst^{(1)}_b) = 1/3$}
}
child {
  node (j) {$\mvhintree^0(\mvconst^{(1)}_{c}) = 1/2$}
  child {
    node (k) {$\mvhintree^1(\mvconst^{(2)}_{ce}) = 1/8$}
  }
  child {
    node (l) {$\mvhintree^1(\mvconst^{(2)}_{cf}) = 5/8$}
      child {
        node (o) {$\mvhintree^2(\mvconst^{(3)}_{cfg}) = 2/5$}
          child { node {$\vdots$} }
          child { node {$\vdots$} }
      }
      child {
        node (p) {$\mvhintree^2(\mvconst^{(3)}_{cfh}) = 3/5$}
          child { node {$\vdots$} }
          child { node {$\vdots$} }
      }
  }
}
;
\end{tikzpicture}
    \caption{An illustration that shows how renormalization affects example beliefs. When transitioning from depth $1$ to depth $2$, $\mvconst^{(1)}_b$ becomes an unsupported constituent so $\mvconst^{(0)}_\epsilon$ is a $1$-renormalization point as it is the closest constituent with supported descendants at depth $2$. The $1/3$ belief assigned to $\mvconst^{(1)}_b$ is redistributed according to Bayes rule across the $1$-renormalization point's descendants at depth $2$ (\ie, $\mvconst^{(1)}_a$ and $\mvconst^{(1)}_c$). Thus $\mvhintree^1 = \renorm_{\mvconst^{(1)}_b}(\mvhintree^0)$. When transitioning from depth $2$ to depth $3$, $\mvconst^{(1)}_{ad}$ and $\mvconst^{(1)}_{ce}$ become unsupported constituents, so $\mvconst^{(0)}_\epsilon$ is a $2$-renormalization point. Thus $\mvhintree^2 = \renorm_{\mvconst^{(2)}_{ad}} \circ \renorm_{\mvconst^{(2)}_{ce}}(\mvhintree^1)$.}
    \label{fig:repr:ht}
\end{figure}

Renormalization of HTs have the following properties.
\begin{proposition}
Suppose $\mvhintree$ is a HT. Then the following holds:
\begin{description}[noitemsep]
    \item[Coherence] $\renorm_{\mvconst^{(d)}_-}(\mvhintree)$ is coherently constructed provided there is some $\mvconst^{(d)} \neq \mvconst^{(d)}_-$ such that $\mvhintree(\mvconst^{(d)}) > 0$;
    \item[Preservation]
    \[
    \mvhintree(\rho_{\mvhintree, \mvconst^{(d)}_-}) = \sum_{\mvconst^{(r+e)} \in \expand(e, \rho_{\mvhintree, \mvconst^{(d)}_-})} \renorm_{\mvconst^{(d)}_-}(\mvhintree)(\mvconst^{(r+e)}) \,.
    \]; and
    \item[Commutative] $\renorm_{\mvconst^{(d)}_2} \circ \renorm_{\mvconst^{(d)}_1}(\mvhintree) = \renorm_{\mvconst^{(d)}_1} \circ \renorm_{\mvconst^{(d)}_2}(\mvhintree)$ provided there is some $\mvconst^{(d)} \neq \mvconst^{(d)}_1$ and $\mvconst^{(d)} \neq \mvconst^{(d)}_2$ such that $\mvhintree(\mvconst^{(d)}) > 0$.
\end{description}
\label{prop:repr:renorm}
\end{proposition}
\begin{proof}
See Section~\ref{subsec:repr:supp} as the proof is straightforward but tedious.
\end{proof}
\noindent The coherence property indicates that renormalization is appropriately defined. Preservation localizes the renormalization to the descendants of a $d$-renormalization point. Commutativity of renormalization means that the order in which we renormalize does not matter and can be interpreted as \emph{exchangeability}. We write $\renorm_{T^d} \eqdef \renorm_{\mvconst^{(d)}_n} \circ \dots \circ \renorm_{\mvconst^{(d)}_1}$ for any $T^d = \set{\mvconst^{(d)}_1, \dots, \mvconst^{(d)}_n}$. Figure~\ref{fig:repr:ht} illustrates how renormalization affects example beliefs.

\paragraph{A process for converging to a depth HT}
Although a depth HT is not computable, there is a process of converting a \emph{reasonable} HT into a depth HT via a sequence of renormalizations. Intuitively, we will obtain a depth HT after we refute every inconsistent constituent. We formalize this process now.

We say that a HT $\mvhintree$ is \emph{reasonable} if $\mvhintree(\mvconst^{(d)}) > 0$ whenever $\mvconst^{(d)}$ is not trivially inconsistent. Put another way, a HT is reasonable if an agent does not assign zero probability to a constituent that it cannot refute using a test for trivial inconsistency.
\begin{example}
Both an uninformative HT and a depth HT are reasonable.
\end{example}

Let $\consts^{(d)}_-$ be the set of depth $d$ constituents that are trivially inconsistent. Define a sequence of HTs $(\mvhintree^d)_{d \in \N}$ inductively as
\begin{align}
\begin{split}
\mvhintree^1 & \eqdef \renorm_{\consts^{(1)}_-}(\mvhintree^0) \\
\mvhintree^{d+1} & \eqdef \renorm_{\consts^{(d+1)}_-}(\mvhintree^d)
\end{split}
\label{eq:htseq}
\end{align}
where $\mvhintree^0$ is some initial HT. Recall that there are consistent constituents at every depth and that constituent constituents have consistent expansions by the constituent completeness theorem so that the sequence of renormalizations is well-defined (see Footnote~\ref{footnote:repr:ht:wd}). The idea is that $\lim_{d \to \infty} \mvhintree^{d}$ converges to a depth HT.

First, we check that renormalization results in a reasonable HT.
\begin{proposition} Let $\mvhintree$ be a reasonable HT. Then
\begin{enumerate}[noitemsep]
    \item $\renorm_{\mvconst^{(d)}}(\mvhintree)$ is reasonable when $\mvconst^{(d)} \in \consts^{(d)}_-$;
    \item $\renorm_{\consts^{(d)}}(\mvhintree)$ is reasonable; and
    \item each $\mvhintree^d$ in the sequence $(\mvhintree^d)_{d \in \N}$ is reasonable.
\end{enumerate}
\label{prop:repr:ht:renorm:reasonable}
\end{proposition}
\begin{proof}
\begin{enumerate}[noitemsep]
    \item We check that $\mvconst^{(e)} > 0$ whenever $\mvconst^{(e)}$ is not trivially inconsistent. The result follows by a straightforward case analysis on whether $\mvconst^{(e)} \in \desc(\rho_{\mvhintree, \mvconst^{(d)}_-})$, $\mvconst^{(e)} \in \anc(\rho_{\mvhintree, \mvconst^{(d)}_-})$, or $\mvconst^{(e)} = \rho_{\mvhintree, \mvconst^{(d)}_-}$. 
    \item The result follows by $|\consts^{(d)}_-|$ applications of Proposition~\ref{prop:repr:ht:renorm:reasonable}, item $1$.
    \item By induction on $d$.
\end{enumerate}
\end{proof}

Second, we check that the limit exists. Roughly speaking, the limit exists because of Bayesian refute and rescale dynamics: either belief in a constituent is refuted and belief in all of its refinements converges to zero or the belief in a constituent is rescaled by belief lost in refuted constituents so that belief in all of its refinements is a monotonically increasing and bounded sequence. We say that $\mvconst^{(d)}$ is \emph{eventually unsupported} with respect to $\mvhintree$ if there exists an $e \in \N$ such that all of its depth $d + e$ expansions $\mvconst^{(d+e)}$ have $\mvhintree(\mvconst^{(d+e)}) = 0$. We say that $\mvconst^{(d)}$ is always supported otherwise.
\begin{proposition}
Let $(\mvhintree^d)_{d \in \N}$ be a sequence of HTs defined as in Equation~\ref{eq:htseq} where $\mvhintree^0$ is a reasonable HT. Then
\begin{enumerate}[noitemsep]
    \item $\lim_{e \to \infty} \mvhintree^e(\mvconst^{(d)}) = 0$ when $\mvconst^{(d)}$ is eventually unsupported with respect to some $\mvhintree^E$ in the sequence; and
    \item $\mvconst^{(d)}$ is eventually unsupported with respect to some $\mvhintree^E$ in the sequence iff it is inconsistent; and
    \item $(\mvhintree^e(\mvconst^{(d)}))_{e \in \N}$ is a monotonically increasing sequence bounded by $1$ when $\mvconst^{(d)}$ is always supported so that $\lim_{e \to \infty} \mvhintree^e(\mvconst^{(d)})$ exists; and
    \item
    \[
    \lim_{e \to \infty} \mvhintree^e(\mvconst^{(d)}) = \sum_{\mvconst^{(d+1)} \expansion \mvconst^{(d)}} \lim_{e \to \infty} \mvhintree^e(\mvconst^{(d+1)}) \,.
    \]
\end{enumerate}
\label{prop:repr:ht:renorm:limit}
\end{proposition}
\begin{proof}\hfill
\begin{enumerate}[noitemsep]
    \item If $\mvconst^{(d)}$ is eventually unsupported, then there is an $E \in \N$ such that
    \[
    \mvhintree^E(\mvconst^{(d)}) = 0 \,.
    \]
    Moreover $\mvhintree^{e}(\mvconst^{(d)}) = 0$ for any $e \geq E$ as renormalization cannot rescale a probability $0$ assignment. Hence the series converges and is $0$.
    \item In the forward direction, we have that there is some $E \in \N$ such that $\mvhintree^E(\mvconst^{(d)}) = 0$ whenever $\mvconst^{(d)}$ is eventually unsupported by the above. Moreover $\mvhintree^E$ is reasonable by Proposition~\ref{prop:repr:ht:renorm:reasonable} so $\mvconst^{(d)}$ is trivially inconsistent. The forward direction follows as a constituent is inconsistent if it is trivially inconsistent.
    
    In the reverse direction, we have that there is some depth $E \in \N$ such that all of the refinements of $\mvconst^{(d)}$ are trivially inconsistent at depth $E$ by the constituent completeness theorem. Thus $\mvhintree^E(\mvconst^{(d)}) = 0$ as $\mvhintree^E$ is reasonable by Proposition~\ref{prop:repr:ht:renorm:reasonable} and the result follows.
    \item Observe that $\mvhintree^e(\mvconst^{(d)}) \leq \mvhintree^{e+1}(\mvconst^{(d)})$ by the preservation property of renormalization (Proposition~\ref{prop:repr:renorm}) when $\mvconst^{(d)}$ is always supported. That we have a monotonically increasing sequence follows by induction on $e$. The sequence is bounded by $1$ because a HT is normalized at every depth. Thus the limit exists. 
    \item We have
    \begin{align*}
        \lim_{e \to \infty} \mvhintree^e(\mvconst^{(d)}) & = \lim_{e \to \infty} \sum_{\mvconst^{(d+1)} \expansion \mvconst^{(d)}} \mvhintree^e(\mvconst^{(d+1)}) \\
        & = \sum_{\mvconst^{(d+1)} \expansion \mvconst^{(d)}} \lim_{e \to \infty} \mvhintree^e(\mvconst^{(d+1)})
    \end{align*}
    where the first equality follows by definition and the second equality follows because the sequence is dominated by $1$ (by items $1$ and $3$).
\end{enumerate}
\end{proof}

We can show the desired result now.
\begin{proposition}
$\lim_{d \to \infty} \mvhintree^{d}$ exists and is a depth HT when $\mvhintree^0$ is a reasonable HT.\footnote{\citet{hintikka1970surface} gives an analogous result where depth information (\ie, a depth HT) is the limit of surface information (\ie, the limit of propagating trivial inconsistency via renormalization).}
\end{proposition}
\begin{proof}
We have that $\mvhintree^\infty \eqdef \lim_{d \to \infty} \mvhintree^{d}$ exists by Proposition~\ref{prop:repr:ht:renorm:limit}, items $1$ and $3$. Next, we check that $\mvhintree^\infty$ is a HT. We clearly have that $\mvhintree^\infty$ satisfies unitial initial beliefs. We have that $\mvhintree^\infty$ is coherently constructed by Proposition~\ref{prop:repr:ht:renorm:limit} item $4$. Finally, observe that $\lim_{e \to \infty} \mvhintree^e(\mvconst^{(d)}) = 0$ iff $\mvconst^{(d)}$ is inconsistent. Finally, observe that $\mvconst^{(d)}$ is eventually unsupported iff it is inconsistent by Proposition~\ref{prop:repr:ht:renorm:limit}, item $2$. Thus $\mvhintree^\infty$ is a depth HT.
\end{proof}

\subsection{Probabilities on First-Order Sentences}
\label{subsec:repr:prob}

As every depth $d$ first-order sentence can be written as a depth $d$ dnf, a HT induces a probability distribution on the validity of first-order sentences. Notably, the distribution does not enforce that logically equivalent statements are assigned the same probability. This means that we can represent the beliefs of an agent that is not logically omniscient. Although logical omniscience fails, the induced distribution does not assign probabilities to logically related sentences arbitrarily.

We begin by defining a topology\footnote{For background on topology, we refer the reader to~\citep{munkres2000topology}.} on the refinement tree. Let $\Psi^d \eqdef \set{\mvconst^{(0)} \dots \mvconst^{(d)} \ST \mvconst^{(0)} \expandsinto \dots \expandsinto \mvconst^{(d)}}$ be the set of length $d$ paths of the refinement tree. Let $\Psi^\omega \eqdef \set{\mvconst^{(0)} \mvconst^{(1)} \dots \ST \mvconst^{(0)} \expandsinto \mvconst^{(1)} \expandsinto \dots}$ be the set of infinite paths of the refinement tree. We write $\mvconst^{(0)} \expandsinto \dots \expandsinto \mvconst^{(d)} \sqsubseteq \mvpath$ if $\mvconst^{(0)} \expandsinto \dots \expandsinto \mvconst^{(d)}$ appears as a finite prefix of $\mvpath \in \Psi^\omega$. Let $\Psi^\omega_{\mvconst^{(d)}} \eqdef \set{ \mvconst^{(d+1)} \mvconst^{(d+2)} \dots \ST \mvconst^{(d)} \expandsinto \mvconst^{(d+1)} \expandsinto \mvconst^{(d+2)} \expandsinto \dots}$ be the set of infinite paths of the refinement tree starting with a refinement of $\mvconst^{(d)}$. Define the topological space $(\Psi^\omega, \cO)$ where $\cO$ is a topology generated by the basis of open sets
\[
\cB \eqdef \set{\mvconst^{(0)} \dots \mvconst^{(d)} \Psi^\omega_{\mvconst^{(d)}} \ST \mvconst^{(0)} \dots \mvconst^{(d)} \in \Psi^d} \cup \set{\emptyset} \,.
\]
Each basic open $\mvconst^{(0)} \dots \mvconst^{(d)} \Psi^\omega_{\mvconst^{(d)}}$ contains every infinite refinement path that begins with $\mvconst^{(0)} \dots \mvconst^{(d)} \in \Psi^d$.
\begin{definition}
The \emph{belief} $\belief_\cB : \cB \rightarrow [0, 1]$ in a basic open $\mvconst^{(0)} \dots \mvconst^{(d)} \Psi^\omega_{\mvconst^{(d)}} \in \cB$ with respect to a HT $\mvhintree$ is defined as
\begin{align*}
    \belief_\cB(\mvconst^{(0)} \dots \mvconst^{(d)} \Psi^\omega_{\mvconst^{(d)}}) & \eqdef \mvhintree(\mvconst^{(d)}) \\
    \belief_\cB(\emptyset) & \eqdef 0 \,.
\end{align*}
\end{definition}
\begin{proposition} The basic opens have consistent assignments:
\[
\belief_\cB(\mvconst^{(0)} \dots \mvconst^{(d)} \Psi^\omega_{\mvconst^{(d)}}) = \sum_{\mvconst^{(d+1)} \expansion \mvconst^{(d)}} \belief_\cB(\mvconst^{(0)} \dots \mvconst^{(d)} \mvconst^{(d+1)} \Psi^\omega_{\mvconst^{(d+1)}}) \,.
\]
\end{proposition}
\label{prop:repr:consistent}
\begin{proof}
This follows directly from the fact that $\mvhintree$ is coherently constructed.
\end{proof}
\noindent Thus we have a finitely additive set function. It is easy to see that $\belief_\cB(\Psi^\omega) = 1$.

We extend the belief in a basic open to the measurable space\footnote{For background on measure-theoretic probability, we refer the reader to~\citep{kallenberg2006foundations}.} $(\Psi^\omega, \sigma(\cO))$ where $\sigma(\cO)$ is the Borel $\sigma$-algebra obtained in the standard way.
\begin{proposition}
The belief $\belief_\cB$ in a basic open defines a unique probability measure $\beta$ on the measurable space $(\Psi^\omega, \sigma(\cO))$.
\end{proposition}
\begin{proof}
Observe that $(\Psi^\omega, \cO)$ has a countable basis $\cB$ so that the Borel $\sigma$-algebra is generated by the basis $\cB$. Moreover, the basis $\cB$ is a $\pi$-system (\ie, closed under finite intersections). The result follows as a finitely additive set function on a $\pi$-system (Proposition~\ref{prop:repr:consistent}) can be uniquely extended to a set function on a $\sigma$-algebra when it is $\sigma$-finite.
\end{proof}

Finally, we define a distribution on first-order sentences.
\begin{definition}
The belief in the validity of first-order sentences is given by
\[
\belief(\mvform^{(d)}) \eqdef \sum_{\mvconst^{(d)} \in \dnf(\mvform{(d)})} \beta(\mvconst^{(0)} \dots \mvconst^{(d)} \Psi^\omega_{\mvconst^{(d)}})
\]
where $\beta$ is the probability measure obtained from $\belief_\cB$.
\end{definition}
\noindent The belief in a first-order formula of depth $d$ with $k$ free variables is the $(d + k)$-belief in the closed first-order formula obtained via universal closure (which increases the depth to $d + k$).

We check that there are HTs that induce probability distributions that do not enforce logical omniscience.
\begin{proposition}[Failure of logical omniscience]
There is a HT $\mvhintree$ such that $\mvform_1 \equiv \mvform_2$ but $\belief(\mvform_1) \neq \belief(\mvform_2)$
\end{proposition}
\begin{proof}
Let $\mvhintree$ be an uninformative HT. Pick any two inconsistent constituents $\mvconst^{(d)}_1$ and $\mvconst^{(d)}_2$. Then $\mvconst^{(d)}_1 \equiv \mvconst^{(d)}_1 \land \mvconst^{(d)}_2$ but $\belief(\mvconst^{(d)}_1) \neq \belief(\mvconst^{(d)}_1 \land \mvconst^{(d)}_2)$.
\end{proof}

Although logical omniscience fails, we cannot assign probabilities arbitrarily. The following proposition highlights some constraints on the probability assignments.
\begin{proposition} The probability on first-order sentences has the following properties:
\begin{enumerate}[noitemsep]
    \item $\belief(\lnot \mvform) = 1 - \belief(\mvform)$;
    \item $\belief(\mvform_1 \land \mvform_2) \leq \min(\belief(\mvform_1), \belief(\mvform_2))$;
    \item $\max(\belief(\mvform_1), \belief(\mvform_2)) \leq \belief(\mvform_1 \lor \mvform_2)$;
    \item
    \[
    \belief((\forall x) \mvform) \leq \min_{\mvconst \in \dnf((\exists x) \lnot \mvform)} \set{1 - \belief(\mvconst)}) \,\mbox{; and}
    \]
    \item
    \[
    \max_{\mvconst \in \dnf((\exists x) \mvform)} \set{ \belief(\mvconst) } \leq \belief((\exists x) \mvform) \,.
    \]
\end{enumerate}
\end{proposition}
\begin{proof}
These all follow from set-theoretic manipulations.
\end{proof}
\noindent For the case of universal and existential quantification, the minimum and maximum are taken over constituents, \ie, possible kinds of individuals, as opposed to individuals in the domain of quantification. Note that this differs with the Gaifman condition~\citep[\eg, see][]{gaifman1964concerning} which defines the probability of a universal or existential as the infimum or supremum over individuals in the domain.

The beliefs possessed by a logically omniscient agent are not computable, and assign probability one to logically valid statements and probability zero to logically invalid statement.
\begin{proposition}\hfill
\begin{enumerate}[noitemsep]
    \item The beliefs with respect to a depth HT satisfy $\belief(\mvform) = 1$ when $\vDash \mvform$ and $\belief(\mvform) = 0$ when $\nvDash \mvform$.
    \item Depth beliefs are not computable.
\end{enumerate}
\end{proposition}
\begin{proof}\hfill
\begin{enumerate}[noitemsep]
    \item When $\nvDash \mvform^{(d)}$, then the dnf of $\mvform^{(d)}$ contains only inconsistent constituents so that $\belief(\mvform^{(d)}) = 0$. To see that $\belief(\mvform^{(d)}) = 1$ when $\vDash \mvform^{(d)}$, recall a formula $\mvform^{(d)}$ is logically valid iff its dnf contains all consistent constituents at depth $d$. By the normalization property of a HT, we have that the $\sum_{\mvconst^{(d)} \mbox{ consistent}} \mvhintree(\mvconst^{(d)}) = 1$ so that $\belief(\mvform^{(d)}) = 1$ when $\vDash \mvform^{(d)}$ as required.
    \item Suppose for the sake of contradiction that depth beliefs are computable. As a constituent is eventually unsupported if it is inconsistent and always supported if it is consistent, we thus have a decision procedure for validity of first-order logic, a contradiction.
\end{enumerate}
\end{proof}

\subsection{An Embedding Space for First-Order Logic}
\label{subsec:repr:embed}

In this section, we embed first-order statements into an associated Hilbert space where mutual exclusion in logic appears as orthogonality in the space. Once we embed first-order statements, we will be able to relate certain logical operations on sentences with operators in the space. As some probabilistic operations can be interpreted as operators, we will also obtain probabilistic analogues of logical operations.

The Hilbert space we choose for the embedding is the standard one obtained by considering square integrable functions over a measurable space.\footnote{For more background on functional analysis, we refer the reader to~\citep{bachman2000fa}.} Let $L^2(\Psi^\omega, \beta)$ be the (weighted) $L^2$ space associated with the probability space $(\Psi^\omega, \sigma(\cO), \beta)$.
\begin{definition}
We have that $L^2(\Psi^\omega, \beta)$ is the Hilbert space \emph{associated} with $\beta$ where the inner product $\inner{\cdot}{\cdot}: L^2(\Psi^\omega, \beta) \times L^2(\Psi^\omega, \beta) \rightarrow \R$ is given by
\[
\inner{f}{g} = \int f \cdot \bar{g} \, d\beta
\]
for $f, g \in L^2(\Psi^\omega, \beta)$ (\ie, $f$ and $g$ are square integrable) and $\bar{g}$ denotes the complex conjugate.\footnote{Because the codomain is $\R$ in our case, the complex conjugate acts as an identity.} As usual, the inner product induces a norm $\lVert \cdot \rVert: L^2(\Psi^\omega, \beta) \rightarrow \R$ where $\lVert f \rVert = \sqrt{\inner{f}{f}}$.
\end{definition}
\noindent As notation, we will overload normal arithmetic operations on numbers to mean their pointwise counterparts on functions. For example, $f + g \eqdef x \mapsto f(x) + g(x)$. We use the infix operator $\oplus$ to take the maximum of two functions: $f \oplus g \eqdef x \mapsto \max(f(x), g(x))$. Similarly, we use the infix operator $\ominus$ to take the minimum of two functions: $f \ominus g \eqdef x \mapsto \min(f(x), g(x))$. 

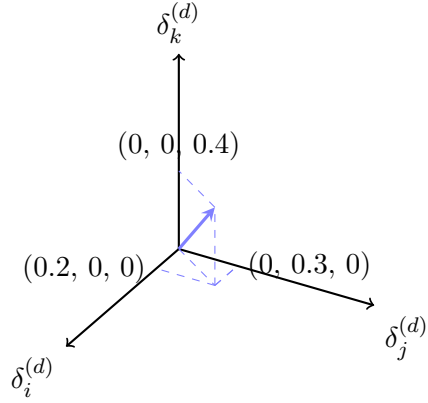
\begin{figure}[t]
    \centering
    \tdplotsetmaincoords{60}{120} 
\begin{tikzpicture} [scale=3, tdplot_main_coords, axis/.style={->,thick}, 
vector/.style={-stealth,blue!50,very thick}, 
vector guide/.style={dashed,blue!50}]

\coordinate (O) at (0,0,0);


\pgfmathsetmacro{\ax}{0.2}
\pgfmathsetmacro{\ay}{0.3}
\pgfmathsetmacro{\az}{0.4}

\coordinate (P) at (\ax,\ay,\az);

\draw[axis] (0,0,0) -- (1,0,0) node[anchor=north east]{$\mvconst^{(d)}_i$};
\draw[axis] (0,0,0) -- (0,1,0) node[anchor=north west]{$\mvconst^{(d)}_j$};
\draw[axis] (0,0,0) -- (0,0,1) node[anchor=south]{$\mvconst^{(d)}_k$};

\draw[vector] (O) -- (P);

\draw[vector guide]         (O) -- (\ax,\ay,0);
\draw[vector guide] (\ax,\ay,0) -- (P);
\draw[vector guide]         (P) -- (0,0,\az);
\draw[vector guide] (\ax,\ay,0) -- (0,\ay,0);
\draw[vector guide] (\ax,\ay,0) -- (0,\ay,0);
\draw[vector guide] (\ax,\ay,0) -- (\ax,0,0);
\node[tdplot_main_coords,anchor=east]
at (\ax,0,0){(\ax, 0, 0)};
\node[tdplot_main_coords,anchor=west]
at (0,\ay,0){(0, \ay, 0)};
\node[tdplot_main_coords,anchor=south]
at (0,0,\az){(0, 0, \az)};
\end{tikzpicture}
    \caption{An illustration of the embedding of $\dnf(\mvconst^{(d)}) = \mvconst^{(d)}_i \lor \mvconst^{(d)}_j \lor \mvconst^{(d)}_j$ where we restrict attention to sentences of depth $d$ or less and we have $\beta(\mvconst^{(0)} \dots \mvconst^{(d)}_i \Psi^\omega_{\mvconst^{(d)}_i}) = 0.2$, $\beta(\mvconst^{(0)} \dots \mvconst^{(d)}_j \Psi^\omega_{\mvconst^{(d)}_j}) = 0.3$, and $\beta(\mvconst^{(0)} \dots \mvconst^{(d)}_k \Psi^\omega_{\mvconst^{(d)}_k}) = 0.4$. Observe that $\belief(\mvform^{(d)}) = 0.8$ so that $\mvform^{(d)}$ is an independent statement with respect to our current beliefs $\belief$.}
    \label{fig:repr:hilbert}
\end{figure}

We embed first-order statements into $L^2(\Psi^\omega, \beta)$ using the intuition that every first-order statement can be written as a finite disjunction of constituents, \ie, mutually exclusive or ``orthogonal" possibilities.
\begin{definition}
Define an embedding into $L^2(\Psi^\omega, \beta)$ as
\[
\mvform^{(d)} \mapsto \sum_{\mvconst^{(d)} \in \dnf(\mvform^{(d)})} \chi_{\mvconst^{(0)} \dots \mvconst^{(d)} \Psi^\omega_{\mvconst^{(d)}}}(\cdot)
\]
where $\chi_X(\cdot)$ is the characteristic function over the set $X$. We write the corresponding element of $\mvform^{(d)} \in \cL$ as $\bm{\mvform^{(d)}} \in L^2(\Psi^\omega, \beta)$.
\end{definition}
\noindent When we consider sentences with maximum depth $D$, then each
\[
\frac{1}{\sqrt{\beta(\mvconst^{(0)} \dots \mvconst^{(D)} \Psi^ \omega_{\mvconst^{(D)}})}} \chi_{\mvconst^{(0)} \dots \mvconst^{(D)} \Psi^\omega_{\mvconst^{(D)}}}(\cdot)
\]
is a basis vector when $\beta(\mvconst^{(0)} \dots \mvconst^{(D)} \Psi^\omega_{\mvconst^{(D)}}) > 0$. Indeed, we can interpret the fact that every first-order sentence $\mvform{(d)}$ of depth $d \leq D$ can be written as a dnf as the analog of the fact that every vector in a finite-dimensional vector space can be written as a (finite) linear combination of basis vectors. Figure~\ref{fig:repr:hilbert} gives an illustration of an example embedding. 

Because we have an embedding of first-order sentences into $L^2(\Psi^\omega, \beta)$ and Hilbert spaces admit complete orthonormal sets $\cA$, we can also write every first-order sentence as a sum of elements from $\cA$. We can think of $\cA$ as an orthogonal ``basis" for first-order logic.
\begin{proposition}
Let $\cA$ be a complete orthonormal set for $L^2(\Psi^\omega, \beta)$. Then $|\cA|$ is either (1) finite or (2) countable depending on the distribution $\beta$.
\end{proposition}
\begin{proof}
$\cA$ can be finite-dimensional when $\beta$ has a finite number of atoms that contain probability $1$. In particular, the finite collection of characteristic functions supported on those atoms forms a basis. To see that $|\cA|$ can be countable, observe that there is a measure preserving bijection up to measure zero between $\Psi^\omega$ and $[0, 1]$. The result follows as $L^2([0, 1], \beta)$ admits a countable complete orthonormal set.
\end{proof}
\noindent In the rest of this section, we will only consider $L^2(\Psi^\omega, \beta)$ where $\beta$ is \emph{reasonable}, \ie, when $\beta$ is derived from a reasonable HT $\mvhintree$.
\begin{proposition} The embedding has the following properties.
\begin{description}[noitemsep]
    \item[Negation as orthogonality] We have $\inner{\bm{\mvform}}{\bm{\lnot \mvform}} = 0$.
    \item[Or as maximum] We have that $\bm{\mvform_1 \lor \mvform_2} = \bm{\mvform_1} \oplus \bm{\mvform_2}$ for any $f, g \in L^2(\Psi^\omega, \beta)$.
    \item[Orthogonality implies mutual exclusion] If $\inner{\bm{\mvform_1}}{\bm{\mvform_2}} = 0$ then $\mvform_1 \implies \lnot \mvform_2$ for any sentences $\mvform_1$ and $\mvform_2$.
\end{description}
\end{proposition}
\begin{proof}\hfill
\begin{description}[noitemsep]
    \item[Negation as orthogonality] The dnfs of $\mvform$ and $\lnot \mvform$ are disjoint. Thus $\bm{\mvform}$ and $\bm{\lnot \mvform}$ are supported on disjoint sets so that the inner product is $0$.
    \item[Or as maximum] Let $\mvform_1^{(d_1)}$ and $\mvform_2^{(d_2)}$ and $d = \max(d_1, d_2)$. The result follows by case analysis on whether $\mvconst^{(d)} \in \dnf(\mvconst^{(d_1)}_1) \cap \dnf(\mvconst^{(d_2)}_2)$ or not.
    \item[Orthogonality implies mutual exclusion] Let $\mvform_1^{(d_1)}$ and $\mvform_2^{(d_2)}$ be any two sentences such that $\inner{\bm{\mvform_1^{(d_1)}}}{\bm{\mvform_2^{(d_2)}}} = 0$. Let $d = \max(d_1, d_2)$. For any
    \[
    \mvconst^{(d)} \in \expand(d - d_1, \dnf(\mvform_1^{(d_1)})_ \cap \expand(d - d_2, \dnf(\mvform_2^{(d_2)})) \,,
    \]
    we have that $\beta(\mvconst^{(0)} \dots \mvconst^{(d)} \Psi^\omega_{\mvconst^{(d)}}) = 0$. We conclude that $\mvconst^{(d)}$ is inconsistent as $\beta$ is reasonable. The result follows as the other constituents are mutually exclusive.
\end{description}
\end{proof}
These properties deserve some remarks. The first item shows that a formula and its negation are orthogonal in $L^2(\Psi^\omega, \beta)$. The second item generalizes the first item and shows that orthogonal elements of $L^2(\Psi^\omega, \beta)$ denote possibilities that are logically mutually exclusive. The converse is true when we consider constituents. The third item shows that logical or (unsurprisingly) acts as a maximum (\ie, a join).

We note that quantification has no effect on the Hilbert space representation. For instance, suppose we take a depth $d$ sentence $\mvform^{(d)}$ and introduce a quantifier as to obtain a depth $d+1$ sentence (\eg, by replacing a duplicated mention of a variable with a new quantifier). This corresponds to restricting attention to the subspace of $L^2(\Psi^\omega, \beta)$ spanned by
\[
\set{\bm{\mvconst^{(d+1)}} \ST \mvconst^{(d+1)} \in \dnf(\expand(1, \mvform^{(d)}))} \,,
\]
which is exactly the subspace we would look at when considering the validity of $\mvform^{(d)}$. A similar situation occurs when we eliminate a quantifier (\eg, by eliminating all mentions of a quantified variable with some other quantified variable).

We consider some elementary interactions between probability and first-order logic using the interpretation of probabilistic operations as operators. In particular, we will be able to analyze the relationship between conditioning and implication. 

\paragraph{Correlation}
The correlation $\rho_{\mvform_1, \mvform_2}$ between two sentences $\mvform_1$ and $\mvform_2$ is given as
\[
\rho_{\mvform_1, \mvform_2} \eqdef \frac{\inner{\bm{\mvform_1} - \beta(\mvform_1)}{\bm{\mvform_2} - \beta(\mvform_2)}}{\sqrt{\beta(\mvform_1) (1 - \beta(\mvform_1)) \beta(\mvform_2) (1 - \beta(\mvform_2))}}
\]
when $\beta(\mvform_1) \neq 0$ or $1$ and $\beta(\mvform_2) \neq 0$ or $1$. Thus the correlation between two statements is defined only when we believe them to be independent. As usual, we can interpret the correlation between $\mvform_1$ and $\mvform_2$ as the cosine of the angle between $\bm{\mvform_1}$ and $\bm{\mvform_2}$.

\paragraph{Conditioning}
The conditional expectation of $\bm{\mvform_2}$ with respect to $\bm{\mvform_1}$, written $\E[\bm{\mvform_2} \ST \bm{\mvform_1}]$, is any $\bm{\mvform_1}$-measurable function such that
\[
\int_{\mvform_1} \E[\bm{\mvform_2} \ST \bm{\mvform_1}] \, d\beta = \int_{\mvform_1} \bm{\mvform_2} \, d\beta \,.
\]
As usual, we can interpret conditioning $\bm{\mvform_2}$ on $\bm{\mvform_1}$ as the projection of $\bm{\mvform_2}$ onto $\bm{\mvform_1}$. Observe that $\E[\bm{\mvform_2} \ST \bm{\mvform_1}]$ can be any function when $\mvform_1$ is inconsistent (\ie, $\bm{\mvform_1} = \bm{0}$) and $\beta$ is a depth HT. In logical terms, we have that a contradiction implies any statement. We consider the interaction between conditioning and implication next.

\paragraph{Implication}
As researchers have noted for a long time in the setting of propositional logic, the probability of the implication $\mvform_1 \rightarrow \mvform_2$ is not the same as its corresponding conditional expectation $\E[\bm{\mvform_2} \ST \bm{\mvform_1}]$~\citep[\eg, see][]{dubois1990logical}. The reason given in the propositional setting is that $\mvform_1 \rightarrow \mvform_2 \equiv \lnot \mvform_1 \lor \mvform_2$ computes the union of areas whereas conditional probability computes a ratio of areas (when it exists) so that the two concepts are different. From the perspective of $L^2(\Psi^\omega, \beta)$, we can (somewhat awkwardly) connect the two concepts. As a reminder, we have that
\[
\bm{\mvform_1 \rightarrow \mvform_2} = \bm{\lnot \mvform_1} \oplus \bm{\mvform_2} \,.
\]
We can rewrite this to use conditional expectations as
\[
\bm{\mvform_1 \rightarrow \mvform_2} = \frac{\lVert \bm{\lnot \mvform_1} \rVert}{\lVert \E[\bm{\mvform_2} \ST \bm{\lnot \mvform_1}] \rVert} \E[\bm{\mvform_2} \ST \bm{\lnot \mvform_1}] \oplus \frac{\lVert \bm{\mvform_2} \rVert}{\lVert \E[\bm{\lnot \mvform_1} \ST \bm{\mvform_2}] \rVert} \E[\bm{\lnot \mvform_1} \ST \bm{\mvform_2}] \,.
\]
Observe that we rewrite $\bm{\lnot \mvform_1}$ as the rescaled orthogonal projection of $\bm{\mvform_2}$ onto $\bm{\lnot \mvform_1}$. (Similarly, we rewrite $\bm{\mvform_2}$ as the rescaled orthogonal projection of $\bm{\lnot \mvform_1}$ onto $\bm{\mvform_2}$.) Thus there is a relationship between implication and conditioning although it is not the one we might expect it to be.

\begin{remark}[An embedding space for first-order logic]
From the perspective of machine learning, we can also think of $L^2(\Psi^\omega, \beta)$ as a natural embedding space for representing first-order logic similar to how $\R^d$ has proved to be a useful embedding space for representing natural language~\citep[\eg, see][]{mikolov2013distributed}. In the setting of natural language processing, there are empirical results suggesting that the operations of vector addition and vector subtractions can be used to add and subtract semantic content from word embeddings. In the logical setting, adding a concept encoded by $\mvform_2$ to a concept encoded by $\mvform_1$ corresponds to embedding $\mvform_1 \lor \mvform_2$. Thus we take their maximum as opposed to performing an addition in $L^2(\Psi^\omega, \beta)$.

Suppose we want to subtract a concept $\mvform_2$ from $\mvform_1$. Logically, we would encode this as $\mvform_1 - \mvform_2 \eqdef \mvform_1 \land \lnot \mvform_2$. Then subtraction of semantic content has the following interpretation:
\[
\bm{\mvform_1 - \mvform_2} = \bm{\mvform_1} \ominus \bm{\lnot \mvform_2} \,.
\]
Thus we do not perform the analogous subtraction in $L^2(\Psi^\omega, \beta)$.
\end{remark}

It would be an interesting direction of future work to examine more in depth what can be said about first-order logic from the viewpoint of $L^2(\Psi^\omega, \beta)$. For instance, are there interesting complete orthonormal sets for first-order sentences and what do the corresponding coefficients look like? For now, we simply note that one can embed first-order logic into a Hilbert space. It is not clear to us whether such an embedding is useful although we do think it intriguing that such an embedding exists.

\subsection{Supplementary on Renormalization}
\label{subsec:repr:supp}

This section contains the supplementary proof for the properties of renormalization (Proposition~\ref{prop:repr:renorm}).
\begin{proof}\hfill
\begin{description}[noitemsep]
    \item[Coherence] We show this by case analysis on whether $\mvconst^{(e)} \in \desc(\rho_{\mvhintree, \mvconst^{(d)}_-})$ or $\mvconst^{(e)} \notin \desc(\rho_{\mvhintree, \mvconst^{(d)}_-})$.
    
    Suppose $\mvconst^{(e)} \in \desc(\rho_{\mvhintree, \mvconst^{(d)}_-})$. We have to show that
    \[
    \renorm_{\mvconst^{(d)}_-}(\mvhintree)(\mvconst^{(e)}) = \sum_{\mvconst^{(e+1)} \expansion \mvconst^{(e)}} \renorm_{\mvconst^{(d)}_-}(\mvhintree)(\mvconst^{(e+1)}) \,.
    \]
    We proceed by case analysis on whether $\mvconst^{(e)} \in \cS^-_{\mvhintree, \mvconst^{(d)}_-}$ or $\mvconst^{(e)} \in \cS^+_{\mvhintree, \mvconst^{(d)}_-}$.
    
    In case of the former, we have that $\renorm_{\mvconst^{(d)}_-}(\mvhintree)(\mvconst^{(e)}) = 0$ and
    \[
    \sum_{\mvconst^{(e+1)} \expansion \mvconst^{(e)}} \renorm_{\mvconst^{(d)}_-}(\mvhintree)(\mvconst^{(e+1)}) = 0
    \]
    as required.
    
    In case of the latter, we have that $\renorm_{\mvconst^{(d)}_-}(\mvhintree)(\mvconst^{(e)}) = 0$ and
    \begin{align*}
        \sum_{\mvconst^{(e+1)} \expansion \mvconst^{(e)}} \renorm_{\mvconst^{(d)}_-}(\mvhintree)(\mvconst^{(e+1)}) & = \frac{Z_{\mvhintree, \mvconst^{(d)}_-}}{Z^+_{\mvhintree, \mvconst^{(d)}_-}} \sum_{\mvconst^{(e+1)} \expansion \mvconst^{(e)}} \mvhintree(\mvconst^{(e)})
    \end{align*}
    by rearranging (We need the hypothesis that there is at least one supported constituent, otherwise we divide by zero). The result follows by the coherence of $\mvhintree$.
    
    Suppose $\mvconst^{(e)} \notin \desc(\rho_{\mvhintree, \mvconst^{(d)}_-})$. The only non-trivial case occurs when $\mvconst^{(e)} = \rho_{\mvhintree, \mvconst^{(d)}_-}$. We have
    \begin{align*}
        \sum_{\mvconst^{(e+1)} \expansion \mvconst^{(e)}} \renorm_{\mvconst^{(d)}_-}(\mvhintree)(\mvconst^{(e+1)}) & = \frac{Z_{\mvhintree, \mvconst^{(d)}_-}}{Z^+_{\mvhintree, \mvconst^{(d)}_-}} \sum_{\mvconst^{(e+1)} \in \child(\mvconst^{(e)}) \cap \cS^+_{\mvhintree, \mvconst^{(d)}_-}} \mvhintree(\mvconst^{(e)}) \\
    \end{align*}
    by substituting definitions. The result follows by observing that
    \[
    \sum_{\mvconst^{(e+1)} \in \child(\mvconst^{(e)}) \cap \cS^+_{\mvconst^{(d)}_-}} \mvhintree(\mvconst^{(e)})
    \]
    is exactly $Z^+_{\mvconst^{(d)}_-}$ so the result follows.

    \item[Preservation] By induction on $e$. The base case is trivial. In the inductive case, we have to show that
    \[
    \mvhintree(\rho_{\mvhintree, \mvconst^{(d)}_-}) = \sum_{\mvconst^{(r+e+1)} \in \expand(e+1, \rho_{\mvhintree, \mvconst^{(d)}_-})} \renorm_{\mvconst^{(d)}_-}(\mvhintree)(\mvconst^{(r+e+1)}) \,.
    \]
    Rewriting the right hand side, we obtain
    \begin{align*}
    & \phantom{=} \sum_{\mvconst^{(r+e)} \in \expand(e, \rho_{\mvhintree, \mvconst^{(d)}_-})} \sum_{\mvconst^{(r+e+1)} \expansion \mvconst^{(r+e)}} \renorm_{\mvconst^{(d)}_-}(\mvhintree)(\mvconst^{(r+e+1)}) \\
    & = \sum_{\mvconst^{(r+e)} \in \expand(e, \rho_{\mvhintree, \mvconst^{(d)}_-})} \renorm_{\mvconst^{(d)}_-}(\mvhintree)(\mvconst^{(r+e)})
    \end{align*}
    where the equality follows by coherence (Proposition~\ref{prop:repr:renorm}, coherence). The result follows by the induction hypothesis.

    \item[Commutative] The proof is quite tedious so we give the intuition first: $\renorm$ is commutative is because $\renorm$ applies Bayes rule to rescale a subtree of $(\consts, \xi)$ and that rescaling by Bayes rule is commutative. The proof follows in two parts. First, we show that the two subtrees (\ie, descendants of the two $d$-redistribution point) we apply rescaling to via Bayes rule to are identical no matter which order we apply renormalization in. Second, it suffices to show that the rescaling on the two subtrees is commutative (due to the subtree property of renormalization).
    
    We start with part one. We claim that the two $d$-redistribution points encountered are identical no matter which order we carry the renormalization. We show this by a direct (and tedious) case analysis. Suppose we apply $\renorm_{\mvconst^{(d)}_1}$ first. We perform case analysis on whether (1) $\rho_{\mvhintree, \mvconst^{(d)}_1} \in \anc(\mvconst^{(d)}_1 \sqcup \mvconst^{(d)}_2)$, (2) or $\rho_{\mvhintree, \mvconst^{(d)}_1} = \mvconst^{(d)}_1 \sqcup \mvconst^{(d)}_2$, or (3) $\rho_{\mvhintree, \mvconst^{(d)}_1} \in \desc(\mvconst^{(d)}_1 \sqcup \mvconst^{(d)}_2)$ where $\mvconst^{(d)}_1 \sqcup \mvconst^{(d)}_2$ is the deepest common ancestor of $\mvconst^{(d)}_1$ and $\mvconst^{(d)}_2$.
    
    Consider the first case $\rho_{\mvhintree, \mvconst^{(d)}_1} \in \anc(\mvconst^{(d)}_1 \sqcup \mvconst^{(d)}_2)$. Observe that $\rho_{\renorm_{\mvconst^{(d)}_1}(\mvhintree), \mvconst^{(d)}_2} = \rho_{\mvhintree, \mvconst^{(d)}_1}$. Otherwise, it would contradict that $\rho_{\mvhintree, \mvconst^{(d)}_1}$ has supported children. We need to show that we encounter the same renormalization point applying $\renorm_{\mvconst^{(d)}_2}$ first. To see this, $\rho_{\mvhintree, \mvconst^{(d)}_2} \expansion \mvconst^{(d)}_1 \sqcup \mvconst^{(d)}_2$ contradicts that $\rho_{\mvhintree, \mvconst^{(d)}_1}$ has supported children. Thus $\rho_{\mvhintree, \mvconst^{(d)}_2} \in \anc(\mvconst^{(d)}_1 \sqcup \mvconst^{(d)}_2)$. Thus we conclude that $\rho_{\mvhintree, \mvconst^{(d)}_1} = \rho_{\mvhintree, \mvconst^{(d)}_1}$ because both give the deepest common ancestor with supported children in a tree. Finally, we conclude that $\rho_{\mvhintree, \mvconst^{(d)}_2} = \rho_{\renorm_{\mvconst^{(d)}_2}(\mvhintree), \mvconst^{(d)}_1}$ as required.
    
    Consider the second case $\rho_{\mvhintree, \mvconst^{(d)}_1} = \mvconst^{(d)}_1 \sqcup \mvconst^{(d)}_2$. There are two subcases to consider: either $\rho_{\renorm_{\mvconst^{(d)}_1}(\mvhintree), \mvconst^{(d)}_2} \in \anc(\mvconst^{(d)}_1 \sqcup \mvconst^{(d)}_2)$ or $\rho_{\renorm_{\mvconst^{(d)}_1}(\mvhintree), \mvconst^{(d)}_2} = \mvconst^{(d)}_1 \sqcup \mvconst^{(d)}_2$.
    
    Consider the first subcase $\rho_{\renorm_{\mvconst^{(d)}_1}(\mvhintree), \mvconst^{(d)}_2} \in \anc(\mvconst^{(d)}_1 \sqcup \mvconst^{(d)}_2)$. We conclude that the path from $\rho_{\mvhintree, \mvconst^{(d)}_1}$ to $\mvconst^{(d)}_2$ is the only path that is positively supported after applying $\renorm_{\mvconst^{(d)}_1}(\mvhintree)$. We see that we encounter the same renormalization points applying $\renorm_{\mvconst^{(d)}_2}$ first by performing an even deeper case analysis: either (1) $\rho_{\mvhintree, \mvconst^{(d)}_2} = \rho_{\mvhintree, \mvconst^{(d)}_1})$ which occurs when the path from $\rho_{\mvhintree, \mvconst^{(d)}_1}$ to $\mvconst^{(d)}_2$ is the only path that is positively supported after applying $\renorm_{\mvconst^{(d)}_2}(\mvhintree)$ or (2) $\rho_{\mvhintree, \mvconst^{(d)}_2} = \rho_{\renorm_{\mvconst^{(d)}_1}(\mvhintree), \mvconst^{(d)}_1})$. Thus we conclude that the result holds in this subcase.
    
    Consider the second subcase $\rho_{\renorm_{\mvconst^{(d)}_1}(\mvhintree), \mvconst^{(d)}_2} = \mvconst^{(d)}_1 \sqcup \mvconst^{(d)}_2$. We show that we encounter the same renormalization points applying $\renorm_{\mvconst^{(d)}_2}$ first. Observe that $\rho_{\mvhintree, \mvconst^{(d)}_2} = \mvconst^{(d)}_1 \sqcup \mvconst^{(d)}_2$. Otherwise, it would contradict that $\rho_{\renorm_{\mvconst^{(d)}_1}(\mvhintree), \mvconst^{(d)}_2}$ has supported children. Similarly, observe that $\rho_{\renorm_{\mvconst^{(d)}_2}(\mvhintree), \mvconst^{(d)}_1} = \mvconst^{(d)}_1 \sqcup \mvconst^{(d)}_2$. Thus the result follows.
    
    Consider the third case $\rho_{\mvhintree, \mvconst^{(d)}_1} \in \desc(\mvconst^{(d)}_1 \sqcup \mvconst^{(d)}_2)$. Observe that $\rho_{\renorm_{\mvconst^{(d)}_1}(\mvhintree), \mvconst^{(d)}_2} \expansion \mvconst^{(d)}_1 \sqcup \mvconst^{(d)}_2$ because $\rho_{\mvhintree, \mvconst^{(d)}_1}$ has supported children. There are two subcases to consider: either (1) $\rho_{\renorm_{\mvconst^{(d)}_1}(\mvhintree), \mvconst^{(d)}_2} = \mvconst^{(d)}_1 \sqcup \mvconst^{(d)}_2$ or (2) $\rho_{\renorm_{\mvconst^{(d)}_1}(\mvhintree), \mvconst^{(d)}_2} \in \desc(\mvconst^{(d)}_1 \sqcup \mvconst^{(d)}_2)$.
    
    Consider the first subcase $\rho_{\renorm_{\mvconst^{(d)}_1}(\mvhintree), \mvconst^{(d)}_2} = \mvconst^{(d)}_1 \sqcup \mvconst^{(d)}_2$. We show that we encounter the same renormalization points applying $\renorm_{\mvconst^{(d)}_2}$ first. Observe that $\rho_{\mvhintree, \mvconst^{(d)}_2} = \mvconst^{(d)}_1 \sqcup \mvconst^{(d)}_2$. Otherwise, it would contradict that $\rho_{\renorm_{\mvconst^{(d)}_1}(\mvhintree), \mvconst^{(d)}_2}$ has supported children. Next, we observe that $\rho_{\renorm_{\mvconst^{(d)}_2}(\mvhintree), \mvconst^{(d)}_1} = \rho_{\mvhintree, \mvconst^{(d)}_1}$ as required.
    
    Consider the second subcase $\rho_{\renorm_{\mvconst^{(d)}_1}(\mvhintree), \mvconst^{(d)}_2} \in \desc(\mvconst^{(d)}_1 \sqcup \mvconst^{(d)}_2)$. Observe that the subtrees of $\rho_{\mvhintree, \mvconst^{(d)}_1}$ and $\rho_{\renorm_{\mvconst^{(d)}_1}(\mvhintree), \mvconst^{(d)}_2}$ are non-overlapping so that the result follows. 
    
    Consequently, there are two $d$-redistribution points $\rho_{\mvconst^{(d)}_a}$ and $\rho_{\mvconst^{(d)}_b}$ and three cases to consider to see that renormalization is commutative for part two of the proof: either (1) $\rho_{\mvconst^{(d)}_a}$ and $\rho_{\mvconst^{(d)}_b}$ are not ancestors of each other, (2) $\rho_{\mvconst^{(d)}_a} = \rho_{\mvconst^{(d)}_b}$, or (3) $\rho_{\mvconst^{(d)}_a}$ is an ancestor of $\rho_{\mvconst^{(d)}_b}$ without loss of generality. The first case is straightforward and second case can be seen as a special case of the third.
    
    Consider the third case where $\rho_{\mvconst^{(d)}_a}$ is an ancestor of $\rho_{\mvconst^{(d)}_b}$. We show that the result holds by another (tedious) case analysis. Let $X^\dagger \eqdef X \cup \bigcup_{x \in X} \desc(x)$. We perform a further case analysis on the position of $\mvconst^{(e)}$ with respect to the support function. Note that the renormalization points may be encountered in the same order or different order. If they are encountered in the same order, then the values are obviously identical. Thus we consider the case when they are encountered in a different order. It suffices to consider the case where $\rho_{\mvconst^{(d)}_a}$ is encountered first followed by $\rho_{\mvconst^{(d)}_b}$ by symmetry.
    
    Let $\cD^+_{\rho_{\mvconst^{(d)}_a}} \eqdef D^+_{\mvhintree, \mvconst^{(d)}_a}$ and $\cD^-_{\rho_{\mvhintree, \mvconst^{(d)}_a}} \eqdef \child(\rho_{\mvhintree, \mvconst^{(d)}_a}) \backslash \cD^+_{\rho_{\mvhintree, \mvconst^{(d)}_a}}$. Moreover let
    \[
    \cD^+_{\rho_{\mvhintree, \mvconst^{(d)}_b}} \eqdef D^+_{\renorm_{\mvconst^{(d)}_a}(\mvhintree), \mvconst^{(d)}_b}
    \]
    and
    \[
    \cD^-_{\rho_{\mvconst^{(d)}_b}} \eqdef \child(\rho_{\renorm_{\mvconst^{(d)}_a}(\mvhintree), \mvconst^{(d)}_b}) \backslash \cD^+_{\rho_{\renorm_{\mvconst^{(d)}_a}(\mvhintree), \mvconst^{(d)}_b}} \,.
    \]
    
    The table below summarizes the values assigned to the different regions.
    \begin{center}
    \begin{tabular}{l|l|l}
        Case $\mvconst^{(e)} \in $ & $\rho_{\mvconst^{(d)}_a}$, $\rho_{\mvconst^{(d)}_b}$ & $\rho_{\mvconst^{(d)}_b}$, $\rho_{\mvconst^{(d)}_a}$ \\ \hline
        $(\cD^-_{\rho_{\mvconst^{(d)}_a}})^\dagger$ & $0$ & $0$ \\
        $(\cD^+_{\rho_{\mvconst^{(d)}_a}} \backslash \child(\mvconst^{(d)}_b))^\dagger$ & $\frac{Z_{\mvconst^{(d)}_a}}{Z^+_{\mvconst^{(d)}_a}} \mvhintree(\mvconst^{(e)})$ & $\frac{\bar{Z}_{\mvconst^{(d)}_a}}{\bar{Z}^+_{\mvconst^{(d)}_a}} \mvhintree(\mvconst^{(e)})$ \\
        $(\cD^-_{\rho_{\mvconst^{(d)}_b}})^\dagger$ & $0$ & $0$ \\
        $(\cD^+_{\rho_{\mvconst^{(d)}_b}})^\dagger$ & $\frac{Z_{\mvconst^{(d)}_a}}{Z^+_{\mvconst^{(d)}_a}} \frac{Z_{\mvconst^{(d)}_b}}{Z^+_{\mvconst^{(d)}_b}} \mvhintree(\mvconst^{(e)})$ & $\frac{\bar{Z}_{\mvconst^{(d)}_a}}{\bar{Z}^+_{\mvconst^{(d)}_a}} \frac{\bar{Z}_{\mvconst^{(d)}_b}}{\bar{Z}^+_{\mvconst^{(d)}_b}} \mvhintree(\mvconst^{(e)})$
    \end{tabular}
    \end{center}
    
    After substituting definitions, we see that
    \[
        \frac{Z_{\mvconst^{(d)}_a}}{Z^+_{\mvconst^{(d)}_a}} = \frac{\sum_{\mvconst^{(f)} \in \child(\rho_{\mvconst^{(d)}_a})} \mvhintree(\mvconst^{(f)})}{\sum_{\mvconst^{(f)} \in D^+_{\mvconst^{(d)}_a}} \mvhintree(\mvconst^{(f)})}
    \]
    and
    \[
        \frac{\bar{Z}_{\mvconst^{(d)}_a}}{\bar{Z}^+_{\mvconst^{(d)}_a}} = \frac{\sum_{\mvconst^{(f)} \in \child(\rho_{\mvconst^{(d)}_a})} \renorm_{\mvconst^{(d)}_b}(\mvhintree)(\mvconst^{(f)})}{\sum_{\mvconst^{(f)} \in D^+_{\mvconst^{(d)}_a}} \renorm_{\mvconst^{(d)}_b}(\mvhintree)(\mvconst^{(f)})}
    \]
    are identical. Similarly, we obtain that $\frac{Z_{\mvconst^{(d)}_a}}{Z^+_{\mvconst^{(d)}_a}}$ and $\frac{\bar{Z}_{\mvconst^{(d)}_a}}{\bar{Z}^+_{\mvconst^{(d)}_a}}$ are also identical. Thus the result follows.
\end{description}
\end{proof}

\section{On Conjecturing}
\label{sec:tp}

Although conjecturing does not directly lead to a proof, it is an integral part of proving in practice: we require interesting conjectures to prove or disprove and attempting a proof may lead to interesting conjectures. In this section, we examine conjecturing as (statistical) \emph{model selection}. When we introduce the game, we will see how conjecturing can be applied as a strategy for playing the game (Section~\ref{sec:game}).

A conjecture, in its barest form, is a well-formed mathematical statement that we (1) do not have a proof for and (2) consider ``interesting". The first criterion is obviously necessary. The second criterion is also necessary but is inherently subjective. With these two criterion in mind, we define a conjecturer now.
\begin{definition}
A \emph{conjecturer} is a function
\[
\mvconj : \mathbf{HT}(\cL) \rightarrow \prod_{d \in \N} \permfunc(\setfunc(\consts^{(d)}))
\]
where $\permfunc(X)$ is the set of permutations on the finite set $X$.
\end{definition}
\noindent A conjecturer maps a HT $\mvhintree$ and a depth $d$ to a permutation on the powerset of depth $d$ constituents. A depth $d$ \emph{conjecture} is a depth $d$ dnf, \ie, it is a subset $X \subseteq \consts^{(d)}$ of depth $d$ constituents. By convention, $\emptyset$ corresponds to conjecturing $\false$. A permutation on the powerset of $\consts^{(d)}$ thus provides a ranking of depth $d$ conjectures that we use as a proxy for ranking how interesting depth $d$ conjectures are. We explore how to construct rankings next.

\subsection{``Interesting" as Model Selection}
\label{subsec:conj:ims}

We convert the problem of quantifying how interesting a conjecture is to a model selection problem. Thus we take a statistical viewpoint of conjecturing. We accomplish this in two stages.
\begin{description}[noitemsep]
    \item[Model class] First, we identify each depth $d$ conjecture with a model\footnote{We use model in the statistical sense and not the model-theoretic sense for which have used the word structure instead.} from some class of models $\cH$. That is, we define a surjection $m: \setfunc(\consts^{(d)}) \rightarrow \cH$ from conjectures of any depth to the model class $\cH$.
    \item[Model scoring] Second, we define a scoring function $\mathscr{S}: \cH \rightarrow \R$ for the model class $\cH$, potentially subject to regularization. Given a scoring function $\mathscr{S}: \cH \rightarrow \R$, we can create a ranking on $\cH$ as $h_1 \leq h_2$ when $\mathscr{S}(h_1) \leq \mathscr{S}(h_2)$ with ties broken arbitrarily.
\end{description}
We give an example of a model class and a scoring function, beginning by identifying a subclass of finite distributions as an example model class for conjectures.
\begin{example}[Distribution conjecture class]
Let $\mvhyp_D \eqdef \set{\mvconst^{(d)} \mapsto \mvhintree(\mvconst^{(d)}) \ST \mvconst^{(d)} \in D}$ for $D \in \setfunc(\consts^{(d)})$ and
\[
\mvhyp^\dagger_D = \mvhyp_D \cup \set{* \mapsto 1 - \sum_{(\mvconst^{(d)} \mapsto b) \in \mvhyp_D} b}
\]
be the distribution that adds a unique element $*$ representing the remaining unassigned belief. We call the class
\[
\mathscr{D} \eqdef \set{\mvhyp^\dagger_D \ST D \in \setfunc(\consts^{(D)})}
\]
of finite distributions a \emph{distribution conjecture class}.
\end{example}

Now that we have a model class for conjectures, we can define a corresponding scoring functions for models to rank conjectures.
\begin{example}[Likelihood-entropy scoring]
A \emph{likelihood-entropy} scoring function for the distribution conjecture class $\mathscr{D}$ scores conjectures as a function of their likelihood and entropy: we have
\[
\mathscr{L}(\mvhyp_D^\dagger) \eqdef \frac{c(|D|) \, \ell(\mvhyp_D) }{H^{|D|}} \cdot \begin{cases}
H(\mvhyp_{D^+}^\dagger) & \mbox{when $D^+ \neq \emptyset$} \\
0 & \mbox{otherwise}
\end{cases}
\]
where $D^+ \eqdef \set{\mvconst^{(d)} \ST \mvhyp_D(\mvconst^{(d)}) > 0}$, $H$ is entropy, $H^{|X|}$ is the entropy of the uniform distribution over $|X|$ elements, $\ell(\mvhyp_D) = \sum_{(x \mapsto b) \in \mvhyp_D, x \neq *} b$ is the total probability except for $*$, and $c: \set{1, \dots, |X|} \rightarrow \R^+$ is a positive and concave function over $\set{1, \dots, |X|}$ ordered by $\leq$. We calculate the entropy of the modified distribution $\mvhyp_{D^+}^\dagger$ that only considers the constituents with positive probability. The factor
\[
\frac{c(|D|) \, \ell(\mvhyp_D)}{H^{|D|}}
\]
is a form of regularization. First, recall that the entropy of a uniform distribution increases as the set of support increases so that an unnormalized measure would be biased towards selecting larger dnfs. Thus we normalize by the entropy of the uniform distribution of the appropriate size. Second, we want to encourage dnfs that include enough possibilities. This is what the concave function $c$ achieves. Lastly, we want to ensure that the conjecture captures enough of what we believe to be true. Otherwise, we would be encouraged towards selecting the least believable constituents because these provide the most information. A likelihood-entropy score can be thought of as measuring how ``informative" a conjecture is weighted by how likely it is.
\end{example}

We check the conjectures generated by a conjecturer using a likelihood-entropy scoring function with beliefs $\mvhintree$ under two extremes.
\begin{description}[noitemsep]
    \item[Uninformative] Given an uninformative HT $\mvhintree$ that assigns beliefs uniformly to constituents at every depth, the ranking produced by a likelihood-entropy scoring ranks conjectures is solely a function of the number of constituents in their dnf. This follows directly from the definition of a likelihood-entropy score. 
    \item[Omniscient] Given a depth HT $\mvhintree$, a likelihood-entropy score ranks conjectures at every depth containing consistent constituents higher than conjectures mentioning no consistent constituents. To see this, observe that if a conjecture has no consistent constituents, then $D^+ = \emptyset$ and thus gets assigned score $0$. Moreover, a conjecture that contains consistent constituents has $D^+ \neq \emptyset$ so that it gets assigned positive score. A conjecturer that knows all logical truths will rank true statements according to the regularization factor and higher than any false statement.
\end{description}

\subsection{Top-down versus Bottom-up Regularization}
\label{subsec:conj:tdvbr}

The regularization given by the concave function $c$ in likelihood-entropy scoring can be seen as a form of bottom-up regularization in that we control the size of the conjecture so that it describes just enough possibilities. We can also impose top-down regularization where we control the sizes and form of the first-order sentences in addition to the sizes of their dnfs. The intuition for additionally considering top-down regularization is that we would like the conjectures to be ``compact and structured enough" to write down in addition to describing just enough possibilities. Let $\cK$ be a finite set of first-order sentences, $d_{\text{min}}$ be the minimum depth of formulas in $\cK$, and $D^d_\cK \eqdef \set{ \mvconst^{(d)} \ST \mvconst^{(d)} \in \dnf(\mvform^{(d)}), \mvform^{(d)} \in \cK}$.
\begin{definition}
A $\cK$-regularized conjecturer is a function
\[
\mvconj_\cK: \mathbf{HT}(\cL) \rightarrow \prod_{d \geq d_{\text{min}}} \permfunc(\setfunc(D^d_k)) \,.
\]
\end{definition}
\noindent Thus a conjecturer is a $\cL$-regularized conjecturer. The definition of a $\cK$-regularized conjecturer allows any kind of subset, although it may be useful to use the regularization to restrict the form of the sentences.
\begin{example}
Any finite subset of $\Sigma^0_n$ or $\Pi^0_n$ can be used to form a $\cK$-regularized conjecturer. Recall that these sentences constrain $\forall$ and $\exists$ to occur in an alternating sequence.
\end{example}
\begin{example}
The singleton set $\cK \eqdef \set{\mvconst^{(d)}}$ can be used to form a $\cK$-regularized conjecturer. In particular, such a conjecturer only generates conjectures that are refinements of $\mvconst^{(d)}$.
\end{example}

\section{On Games and Proving}
\label{sec:game}

We introduce an alternating-turn game that involves determining the consistency of constituents. Note that agents are not directly proving theorems so we begin by showing how to construct a prover from beliefs represented by a HT (Section~\ref{subsec:game:btp}). In particular, the prover is complete when logical omniscience is attained (\ie, a depth HT) and sound if the agent maintains reasonable beliefs. We then introduce the game formally (Section~\ref{subsec:game:pathfinder}), examine game play (Section~\ref{subsec:game:play}), and identify how conjecturing fits into game play (Section~\ref{subsec:game:conj}). The game is amenable to self-play training similar to those used to learn Chess and Go, although the challenging task of implementation and empirically testing self-play is beyond the scope of this paper (Section~\ref{subsec:game:self}). One reason for the technical difficulty is that the representation has intractable space requirements. We will comment on how to reduce the space complexity in the next section by using \emph{abstractions} (Section~\ref{sec:abs}).

\begin{remark}[On supervised learning]
We note that we can formulate the learning of a function approximating a depth HT $\mvhintree$ as a supervised learning problem as opposed to constructing a game that is amenable (in principle) to self-play training as we will do in the rest of the section. Although this is possible, we will need to devise a methodology for selecting constituents in the supervised approach. In particular, the typical assumption of independent and identically distributed samples of constituents and their valuations of consistency is not a good one in this setting as constituents can be related to one another via the refinement relation. Indeed, as we will see, agents learn in Pathfinder by considering sequences of \emph{dependent} constituents that are refinements of one another.
\end{remark}

\subsection{From Beliefs to Proofs}
\label{subsec:game:btp}

Before we introduce the game, we explain how to extract a proof attempt from beliefs represented by a HT $\mvhintree$. Define a function $\prove_\mvhintree: \cL \rightarrow \two$ as
\[
\prove_\mvhintree(\mvform^{(d)}) = \begin{cases}
\true & \mbox{$\sum_{\mvconst^{(d)} \in \dnf(\mvform^{(d)})} \mvhintree(\mvconst^{(d)}) = 1$} \\
\false & \mbox{otherwise.}
\end{cases}
\]

We show that the function $\prove_\mvhintree$ converts reasonable beliefs into proofs. 
\begin{proposition}[Soundness and completeness]\hfill
\begin{description}[noitemsep]
    \item[Sound] The procedure $\prove_\mvhintree$ is sound if $\mvhintree$ is reasonable. 
    \item[Complete] The procedure $\prove_\mvhintree$ is complete whenever $\mvhintree$ is a depth HT.
\end{description}
\end{proposition}
\begin{proof}\hfill
\begin{enumerate}[noitemsep]
\item Suppose for the sake of contradiction that $\prove_\mvhintree(\mvform^{(d)}) = \true$ but $\mvform^{(d)}$ is inconsistent. Thus there is at least one consistent constituent $\mvconst^{(d)} \notin \dnf(\mvform^{(d)})$. We conclude that $\mvhintree(\mvconst^{(d)}) = 0$ when $\prove_\mvhintree(\mvform^{(d)}) = \true$ because $\sum_{\mvconst^{(d)} \in \dnf(\mvform^{(d)})} \mvhintree(\mvconst^{(d)}) = 1$ and by the normalization property of HTs. This contracts the assumption that $\mvhintree(\mvconst^{(d)}) > 0$ when it is not trivially inconsistent.
\item Recall that a formula is logically valid iff its dnf contains all consistent constituents.
\end{enumerate}
\end{proof}
\noindent The first part shows that agents are not required to be logically omniscient in order to obtain a sound prover. The second part of the proposition above indicates that we are only at risk of losing completeness. The situation intuitively makes sense: so long as we are not logically omniscient, we will not be able to prove every true theorem. We turn our attention now towards learning beliefs $\mvhintree$.

\subsection{Pathfinder: A Game for Learning Beliefs}
\label{subsec:game:pathfinder}

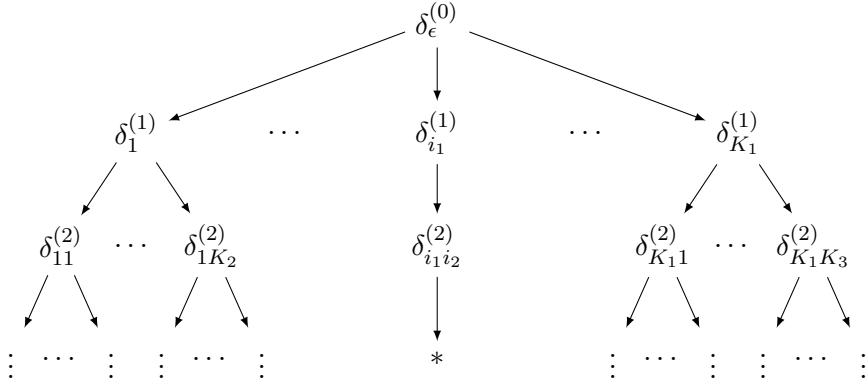
\begin{figure}
    \centering
    \begin{tikzpicture}[edge from parent/.style={draw,-latex},level/.style={sibling distance=40mm/#1}]
\node [] (z){$\mvconst^{(0)}_\epsilon$}
child {
  node (a) {$\mvconst^{(1)}_{1}$}
  child {
    node (b) {$\mvconst^{(2)}_{11}$}
      child { node (c) {$\vdots$} } 
      child { node (d) {$\vdots$} }
  }
  child {
    node (g) {$\mvconst^{(2)}_{1K_2}$}
      child { node (e) {$\vdots$} }
      child { node (f) {$\vdots$} }
  }
}
child {
  node (ab) {$\mvconst^{(1)}_{i_1}$}
  child {
    node (abc) {$\mvconst^{(2)}_{i_1 i_2}$}
    child {
      node {$*$}
    }
  }
}
child {
  node (j) {$\mvconst^{(1)}_{K_1}$}
  child {
    node (k) {$\mvconst^{(2)}_{K_1 1}$}
      child { node (m) {$\vdots$} } 
      child { node (n) {$\vdots$} }
  }
  child {
    node (l) {$\mvconst^{(2)}_{K_1 K_3}$}
      child { node (o) {$\vdots$} }
      child { node (p) {$\vdots$} }
  }
}
;
\path (a) -- (ab) node [midway] {\dots};
\path (ab) -- (j) node [midway] {\dots};
\path (b) -- (g) node [midway] {\dots};
\path (k) -- (l) node [midway] {\dots};
\path (c) -- (d) node [midway] {\dots};
\path (e) -- (f) node [midway] {\dots};
\path (m) -- (n) node [midway] {\dots};
\path (o) -- (p) node [midway] {\dots};
\end{tikzpicture}
    \caption{Example of game flow in Pathfinder. Player one selects $\mvconst^{(1)}_{i_1}$. Player two selects $\mvconst^{(2)}_{i_1 i_2}$. Player one believes that $\mvconst^{(2)}_{i_1 i_2}$ is inconsistent and issues a challenge, thus ending the game.}
    \label{fig:game:pathfinder:flow}
\end{figure}

\emph{Pathfinder} is an alternating-turn game where the goal of the game is to recognize inconsistent constituents. Because we can extract a prover given a HT $\mvhintree$ as above, agents that learn to recognize the consistency of constituents well will learn to be a better theorem prover (\ie, be able to prove more theorems).

A player is given a depth $d$ constituent $\mvconst^{(d)}$ and allowed to make one of two \emph{moves}. 
\begin{enumerate}[noitemsep]
    \item A player can \emph{select} a refinement constituent $\mvconst^{(d+1)} \expansion \mvconst^{(d)}$ and pass play to the other player. The select move introduces an existential which intuitively corresponds to the construction of an auxiliary object that may be useful for the proof.\footnote{As a concrete instance, consider proofs in Euclidean geometry. These proofs involve constructing the appropriate points, lines, and circles so that the conclusion is ``obvious". This method of proof contrasts with the design of many first-order automated theorem provers where quantifiers are lifted to the head of the formula and eliminated.}
    \item A player can issue a \emph{challenge} meaning that the player believes the constituent to be \emph{inconsistent}.\footnote{We emphasize that the condition is \emph{inconsistency} and not \emph{trivial inconsistency}. Thus the challenge move requires an oracle to implement. Naturally, testing for trivial inconsistency up to a certain depth can serve as a proxy test for inconsistency. By the constituent completeness theorem, testing for trivial inconsistency up to ``infinite" depth is equivalent to testing for inconsistency.} If $\mvconst^{(d)}$ is revealed to be inconsistent, the player issuing the challenge wins. Otherwise, if $\mvconst^{(d)}$ is revealed to be consistent, then the player issuing the challenge loses.
\end{enumerate}
In order to play the game well, the players need to develop an intuition about which constituents ``look" inconsistent. Figure~\ref{fig:game:pathfinder:flow} illustrates the flow of an example game. We describe the game more formally now.

Let $*$ represent the terminal state reached after a challenge is issued. Let $X \eqdef \consts \cup \set{*}$ denote the \emph{states} of Pathfinder. We write $x \in X$ to denote a generic state or $\mvconst^{(d)} \in X$ when it is a constituent. Define the \emph{positions} of Pathfinder to be the set
\[
P \eqdef \bigcup_{d \in \N} \seqfunc^d(X) 
\]
of all finite sequences of states.

Let the two players be $O$ for odd and $E$ for even. Define the \emph{turn order} function $T: P \rightarrow \set{O, E}$ as
\[
T(x_1 \dots x_n) = 
\begin{cases}
O & \mbox{$n$ even} \\
E & \mbox{$n$ odd}
\end{cases} \,.
\]
Thus player $O$ plays the positions that have even length (resulting in a position that has odd length) and player $E$ plays the positions that have odd length.

Next, we define transition relation $\rightsquigarrow: P \rightarrow P \rightarrow \bm{2}$ to indicate the legal \emph{moves}. We give the inference rules generating $\rightsquigarrow$ below.
\begin{description}[noitemsep]
    \item[Select] $\mvconst^{(0)} \dots \mvconst^{(d)} \rightsquigarrow \mvconst^{(0)} \dots \mvconst^{(d)} \, \mvconst^{(d + 1)}$ whenever $\mvconst^{(d)} \expandsinto \mvconst^{(d+1)}$
    \item[Challenge] $\mvconst^{(0)} \dots \mvconst^{(d)} \rightsquigarrow \mvconst^{(0)} \dots \mvconst^{(d)} \, *$
\end{description}
The player whose turn it is to move chooses either select or challenge.

Finally, we define the function $W: P \rightarrow \set{O, E}$ which determines which player \emph{wins}:
\[
W(x_1 \dots x_{n-1} \, *) = 
\begin{cases}
O & \mbox{$n$ odd and $x_{n-1}$ inconsistent} \\
E & \mbox{$n$ odd and $x_{n-1}$ consistent} \\
O & \mbox{$n$ even and $x_{n-1}$ inconsistent} \\
E & \mbox{$n$ even and $x_{n-1}$ consistent}
\end{cases} \,.
\]

We can define the game now that we have all the requisite components.
\begin{definition}
The \emph{Pathfinder} game is given by the tuple $\mathfrak{P} \eqdef (P, T, \rightsquigarrow, W)$.
\end{definition}

We emphasize that the challenge move of Pathfinder game play involves determining the inconsistency of first-order statements. Thus it is only semi-decidable. One can implement a modified winning condition $W^e$ which uses the decidable condition of trivial inconsistency instead.
\[
W^e(x_1 \dots x_{n-1} \, *) = 
\begin{cases}
O & \mbox{$n$ odd and $y$ trivially inconsistent} \\
E & \mbox{$n$ odd and $y$ not trivially inconsistent} \\
O & \mbox{$n$ even and $y$ trivially inconsistent} \\
E & \mbox{$n$ even and $y$ not trivially inconsistent}
\end{cases}
\]
where $y = \expand(e, x_{n-1})$. We have that $\lim_{e \to \infty} W^e = W$ by the constituent completeness theorem.

\subsection{Playing Pathfinder}
\label{subsec:game:play}

As we have just seen, the rules for Pathfinder game are quite simple. Nevertheless, like many other games whose rules are easy to state, playing Pathfinder ``well" is difficult because it reduces to determining the consistency of first-order statements. We can analyze the plays made by agents (\ie, what it means to play ``well") using the formalization above. Towards this end, we model an agent as using a HT to guide their game play for Pathfinder.

\begin{algorithm}[t]
  \begin{algorithmic}[1]
    \Function{step}{$\mvconst^{(d)}$}
      \If{$\sum_{\mvconst^{'(d+1)} \geq \mvconst^{(d)}} \mvhintree(\mvconst^{'(d+1)}) = 0$}
        \State challenge
      \Else
        \If{flip($1 - \mvhintree(\mvconst^{(d)})$) = true}
          \State challenge
        \Else
          \State select $\mvconst^{(d+1)} \geq \mvconst^{(d)}$ with probability $\frac{\mvhintree(\mvconst^{(d+1)})}{\sum_{\mvconst^{'(d+1)} \geq \mvconst^{(d)}} \mvhintree(\mvconst^{'(d+1)})}$
        \EndIf
      \EndIf
    \EndFunction
  \end{algorithmic}
  \caption{Strategy for rational agent $A_\mvhintree$.}
  \label{alg:game:ht}
\end{algorithm}

Suppose an agent $A_\mvhintree$ playing Pathfinder uses a HT $\mvhintree$ to represent its beliefs in mathematical statements. Intuitively, we should be able to derive a strategy for playing Pathfinder that is compatible with the agent's beliefs $\mvhintree$. In essence, it should issue challenges and select constituents in proportion to the probability $\mvhintree$ assigns to the consistency of each constituent. More formally, a \emph{strategy} for a player says for each position what the next position to play is when it is that player's turn. We say that the agent $A_\mvhintree$ is \emph{rational} if it plays the strategy given by Algorithm~\ref{alg:game:ht} (hence it plays a \emph{mixed strategy}). In words, the agent first checks that it does not believe all continuations are inconsistent as $\sum_{\mvconst^{'(d+1)} \geq \mvconst^{(d)}} \mvhintree(\mvconst^{'(d+1)}) \neq 0$ and challenges if it is (lines $2$--$3$). If it is not, then the agent challenges with probability $1 - \mvhintree(\mvconst^{(d)})$ (lines $5$--$6$). With the remainder of the probability, it selects a constituent $\mvconst^{(d+1)} \geq \mvconst^{(d)}$ in proportion to its belief in its consistency
\[
\frac{\mvhintree(\mvconst^{(d+1)})}{\sum_{\mvconst^{'(d+1)} \geq \mvconst^{(d)}} \mvhintree(\mvconst^{'(d+1)})} \mbox{(lines $7$--$8$).}
\]

As we might expect, a rational agent with perfect knowledge is able to achieve \emph{optimal play}: (1) only challenge inconsistent constituents and (2) only select consistent constituents.
\begin{proposition}
A rational agent $A_\mvhintree$ where $\mvhintree$ is a depth HT achieves optimal play.
\end{proposition}
\begin{proof}
By assumption, $\mvhintree$ is a depth HT so it assigns inconsistent constituents probability $0$. We proceed by case analysis. $\mvconst^{(d)}$ is inconsistent when $\sum_{\mvconst^{'(d+1)} \expansion \mvconst^{(d)}} \mvhintree(\mvconst^{'(d+1)}) = 0$ so a challenge is issued. If $\mvhintree(\mvconst^{(d)}) = 0$, then the agent challenges with probability $1$. If $\mvhintree(\mvconst^{(d)}) > 0$, then the agent selects only consistent constituents and passes play to the second player.
\end{proof}

The game of Pathfinder continues ad infinitum with optimal play, \ie, is drawn. Recall that a strategy is \emph{winning} if the strategy always produces a win no matter what the other player does.
\begin{proposition}
There are no winning strategies.
\end{proposition}
\begin{proof}
The contrapositive of the completeness theorem for constituents gives that every consistent constituent has a refinement. Hence both players always have a non-losing continuation.
\end{proof}
\noindent Of course, optimal play is not computable as a depth HT is not computable.

We emphasize that Pathfinder game play involves determining the consistency of first-order statements and \emph{not} the logical validity of first-order statements. One consequence of this choice is that there is learning signal in independent statements.
\begin{remark}[Learning signal in independent statements]
An agent chooses between two outputs when playing Pathfinder: (1) inconsistent (\ie, constituent is satisfiable in no models) or (2) consistent (\ie, constituent is satisfiable in at least one model). In particular, note that a consistent constituent at depth $d$ is an independent (\ie, unprovable) statement whenever there are at least two consistent constituents at depth $d$. This follows as a consequence of the mutual exclusivity of any two constituents at depth $d$. Thus there is learning signal in the independent regions of a HT. Put another way, an agent that plays Pathfinder well is incentivized towards playing independent constituents.\footnote{The only time an agent will play a provable constituent is when that constituent is the only consistent constituent at that depth.} This situation differs from a theorem proving setup where agents choose between classifying input statements as (1) inconsistent or (2) logically valid (\ie, satisfiable in every model or provable) so that the learning signal obtained from exploring independent regions is less obvious.
\end{remark}

\subsection{Incorporating Conjecturing}
\label{subsec:game:conj}

\begin{algorithm}[t]
  \begin{algorithmic}[1]
    \Function{step}{$\mvconst^{(d)}$}
      \If{$\sum_{\mvconst^{'(d+1)} \geq \mvconst^{(d)}} \mvhintree(\mvconst^{'(d+1)}) = 0$}
        \State challenge
      \Else
        \If{flip($1 - \mvhintree(\mvconst^{(d)})$) = true}
          \State challenge
        \Else
          \State $\pi \gets \mvconj_{\set{\mvconst^{(d)}}}(\mvhintree)(d+1)$
          \State select $\mvconst^{(d+1)}$ with probability $\frac{\mvhintree(\mvconst^{(d+1)})}{\sum_{\mvconst^{'(d+1)} \in \pi_1, \mvconst^{'(d+1)} \neq *} \mvhintree(\mvconst^{'(d+1)})}$
        \EndIf
      \EndIf
    \EndFunction
  \end{algorithmic}
  \caption{Strategy for conjecturing agent $A_\mvhintree$.}
  \label{alg:game:conj}
\end{algorithm}

We can incorporate conjecturing into the playing of Pathfinder. We say that an agent is a \emph{conjecturing} agent if it plays the (mixed) strategy given in Algorithm~\ref{alg:game:conj}. The conditions for challenging are identical to the ones played by a rational agent. The difference occurs in the selection of the next constituent to play (lines $7$--$9$). In this case, the agent uses a $\cK$-regularized conjecturer where $\cK = \expand(\mvconst^{(d)})$ to generate a ranking of conjectures $\pi$. Next, the agent selects the highest ranked conjecture $\pi_1$ and selects the constituent from that excluding $*$ following the probabilities given by $\mvhintree$. 

\begin{proposition}
A conjecturing agent $A_\mvhintree$ using a likelihood-entropy scoring function where $\mvhintree$ is a depth HT achieves optimal play.
\end{proposition}
\begin{proof}
The only difference is the select case. We claim that $\pi_1$ contains consistent constituents. Assume for the sake of contradiction that it does not. As a likelihood-entropy scoring function ranks conjectures containing consistent constituents higher than those that contain none, then $\pi_1$ contains no consistent constituents. But this means that $\mvconst^{(d)}$ is inconsistent because it contains no refinement constituents that are consistent, a contradiction. As a depth HT assigns inconsistent constituents $0$ belief and $\pi_1$ contains consistent constituents, an agent selecting constituents in proportion to their beliefs will select a consistent constituent as required.
\end{proof}

\subsection{A Note on Self-Play for Pathfinder}
\label{subsec:game:self}

We note that self-play training similar to those described in the literature~\citep[\eg, see][]{tesauro1992practical,silver2016mastering,silver2017mastering} is applicable to Pathfinder as it is an alternating-turn game with symmetric play. We recall the standard setup here to make the idea concrete. Of course, the implementation and empirical testing of self-play setups are the most challenging and non-trivial portions of the task, which we do not address in this paper.

Let $x_1 \dots x_{t}$ be a sequence of Pathfinder positions where $x_t = \mvconst^{(0)} \dots \mvconst^{(d)} \, *$ is a terminal board state. We truncate games so that they take at most $N$ steps. If no challenge is issued within $N$ steps, we say that the game is \emph{drawn}. Define a \emph{reward signal} $z_O$ for player $O$ as $z_O \eqdef 1$ when $O$ wins, $z_O \eqdef -1$ when $O$ loses, and $z_O \eqdef 0$ when there is a draw. As usual, the reward signal $z_E$ for the other player $E$ is the negation $z_E = -z_O$.

Define a parameterized function
\[
f_\theta: \prod_{d \in \N} \Psi^{(d)} \rightarrow \distfunc(\consts^{(d+1)} \cup \set{\text{challenge}}) \times [0, 1]
\]
where $\distfunc(X)$ gives the collection of finite distributions on $X$ which takes a current depth $d$ and a path through the refinement tree, and produces a distribution on constituents to select or to challenge paired with an estimate of the expected value (with respect to the move probabilities) of winning for player $O$ starting at the current path. Suppose we have taken the refinement path $\mvconst^{(0)} \expandsinto \dots \expandsinto \mvconst^{(d)}$ and that $f(\mvconst^{(0)} \dots \mvconst^{(d)}) = (p_1, \dots, p_{K_{d+1}}, p_{\text{challenge}}, \hat{z_O})$. A self-play game can be generated by selecting the move in proportion to $(p_1, \dots, p_{K_{d+1}}, p_{\text{challenge}})$. The hope is to learn the parameters $\theta$ (\eg, via self-play) such that $\pi_1(f_\theta) \approx \mvhintree$ where $\mvhintree$ is a depth HT. Note that we cannot adjust the parameters arbitrarily if we hope to maintain guarantees on the derived prover.
\begin{remark}[Maintaining reasonable beliefs]
It is important in the course of self-play to maintain reasonable beliefs. Notably, if $\mvhintree$ is not reasonable, the resulting prover will not be sound because at least one consistent constituent will be assigned zero weight. Provided that we initialize a self-play agent with reasonable beliefs and ensure that beliefs in constituents are never zeroed unless they are known to be inconsistent, then the agent will maintain reasonable beliefs.
\end{remark}
\begin{remark}[Renormalization as an update rule]
When an agent playing Pathfinder discovers a constituent to be inconsistent, the agent can apply renormalization (see $\renorm$, Section~\ref{subsubsec:repr:ht:renorm}) to refute beliefs in that constituent and all of its descendants and rescale the rest of its beliefs appropriately. Notably, eliminating beliefs in an inconsistent constituent and applying renormalization leaves reasonable beliefs invariant. In this case, we say that $\renorm$ \emph{respects} reasonable beliefs. Naturally, there are other update rules that respect reasonable beliefs. In particular, we can replace rescaling with any other method of redistributing beliefs provided that we distribute the beliefs to the appropriate supported constituents.
\end{remark}

Although Pathfinder is not implementable as presented, we may still wonder in principle how to measure the performance of a system that plays Pathfinder. Obviously, the standard theorem proving setup where one measures the percentage of theorems proved in a benchmark of theorems can be applied in our setting. Given learned beliefs $\mvhintree$, we can measure the percentage of theorems proved by $\prove_\mvhintree$. We can also measure the analog of a ``partial proof" for beliefs. More concretely, we can check that $\prove_\mvhintree(\mvform) > 1 - \epsilon_\true$ for true statements $\mvform \in B$ where $0 < \epsilon_\true < 1$ and $\prove_\mvhintree(\mvform) < \epsilon_\false$ for false statement $\mvform \in B$ where $0 < \epsilon_\false < 1$. Beyond using the beliefs output by a system at a single point in time, we can also consider the evolution of beliefs in statements over time.
\begin{proposition}[Evolution of beliefs]
Let $(\mvhintree^n)_{n \in \N}$ be the sequence of HTs defined in Equation~\ref{eq:htseq} and $(\beta^n)_{n \in \N}$ be the associated sequence of probability assignments to first-order statements. Then $(\beta^n(\mvform^{(d)}))_{n \in \N}$ is eventually constant for any sentence $\mvform^{(d)}$.
\end{proposition}
\begin{proof}
Recall that $\renorm$ only affects the descendants of a $e$-renormalization point. As the dnf of a sentence $\mvform^{(d)}$ only has a finite number of ancestors, the number of $e$-renormalization points is finite. Thus the probability assignment $(\beta^n(\mvform^{(d)}))_{n \in \N}$ is eventually constant.
\end{proof}
\noindent Note that the sequence of probability assignments $(\beta^n(\mvform^{(d)}))_{n \in \N}$ is not guaranteed to be increasing for a logically valid statement or decreasing for an inconsistent statement. Moreover, although the probability assignment for a statement only changes a finite number of times, the number of times that the probability assignment changes is not computable.

\section{On Abstraction}
\label{sec:abs}

Both conjecturing and Pathfinder are not practically implementable as currently presented because there are a super-exponential number of constituents as a function of depth resulting in a HT having a super-exponential branching factor. The reason that there are so many depth $d$ constituents is because they provide the finest grained view of possible kinds of worlds describable with respect to $d$ individuals. However, for most intents and purposes, we can take a coarser grained view that captures the details that we care about. In other words, we can treat certain possibilities as observationally indistinguishable to reduce the space complexity of a dnfs, \ie, make \emph{abstractions} and \emph{lazily} consider more details as needed (Section~\ref{subsec:abs:filt} and Section~\ref{subsec:abs:choose}). At the end of the section, we will introduce \emph{Trailblazer}, a modification of the Pathfinder game, that utilizes abstractions and laziness to trade-off completeness for on-demand space requirements (Section~\ref{subsec:abs:trail}).

\subsection{Filtrations}
\label{subsec:abs:filt}

The basic idea we have in mind is to control the ``resolution" at which constituents distinguish possibilities by partitioning each set of depth $d$ constituents in a compatible manner across depth. Each cell of the partition describes all of the possibilities identified by that cell's member constituents.

Let $\set{C^{(d)}_i}$ be a partition of $\consts^{(d)}$. For each $C^{(d)}_i$, define the \emph{super constituent} $\mvsconst^{(d)}_i$ with respect to a partition $\set{C^{(d)}_i}$ as
\[
\mvsconst^{(d)}_i \eqdef \lOr_{\mvconst^{(d)} \in C^{(d)}_i} \mvconst^{(d)} \,.
\]
A super constituent collapses multiple distinct possibilities into one possibility, and thus, can be viewed as a form of abstraction. Let $\mathbb{S}^{(d)}$ be the set of super constituents with respect to the partition $\set{C^{(d)}_i}$. Naturally, a super constituent is said to be \emph{trivially inconsistent} if all of its members are trivially inconsistent.

\begin{definition}
We say $\cF = (\set{C^{(d)}_i})_{d \in \N}$ where each $\set{C^{(d)}_i}$ is a partition of $\consts^{(d)}$ is a \emph{filtration} of $(\consts, \xi)$ if adjacent elements satisfy the following condition: for every cell $C^{(d)}_j \in \set{C^{(d)}_i}$, there exists a subset $D \subseteq \set{C^{(d+1)}_i}$ such that $C^{(d)}_j = \bigcup_{C^{(d+1)}_k \in D} C^{(d+1)}_k$.
\end{definition}
\noindent In words, we have a filtration if the partition at depth $d+1$ of $\consts^{(d+1)}$ can be used to form a partition of each cell at depth $d$. A filtration induces a corresponding set of super constituents. 

We can lift the refinement partial order on partitions to filtrations. Let $\mathscr{F}$ be the set of all filtrations. We have that $(\mathscr{F}, \sqsubseteq)$ is a partial order where $\cF_1 \sqsubseteq \cF_2$ if each depth $d$ partition in $\cF_2$ is finer than the corresponding depth $d$ partition in $\cF_1$. At one extreme, we have a filtration consisting of one cell that contains every constituent so that it has the lowest resolution. At the other extreme, each filtration assigns each constituent to its own set so that we have the highest resolution possible so that no space savings is gained. We can intuitively think of the ``resolution" of a filtration $\cF$ as the height in the Hasse diagram of $\mathscr{F}$. Naturally, some resolutions are incomparable.

Super constituents possess some of the same properties as constituents.
\begin{proposition}\hfill
\begin{description}[noitemsep]
    \item[Mutually exclusive] Any two super constituents of the same depth are mutually exclusive.
    \item[Expansion] Every depth $d$ super constituent can be written as a disjunction of super constituents of greater depth.
    \item[Completeness] A super constituent is inconsistent if and only if all of its refinements at some depth are trivially inconsistent.
\end{description}
\end{proposition}
\begin{proof}
These properties all follow directly from the properties of partitions.
\end{proof}

In general, we lose existence of super constituents: there are depth $d$ sentence $\mvform^{(d)}$ that cannot be written as a disjunction of depth $d$ super constituents. For example, the super constituents obtained from the trivial filtration cannot express logically invalid statements. We say that a filtration is \emph{complete} if it assigns every consistent constituent to its own cell.
\begin{proposition}[Complete existence]
Every depth $d$ sentence can be written as a disjunction of depth $d$ super constituents given by a complete filtration $(\set{C^{(d)}})_{d \in \N}$.
\end{proposition}
\begin{proof}
Recall that we can adjoin inconsistent constituents to a dnf without affecting its satisfiability.
\end{proof}
\noindent As we might expect by now, a complete filtration is not computable.

\subsection{Choosing Filtrations}
\label{subsec:abs:choose}

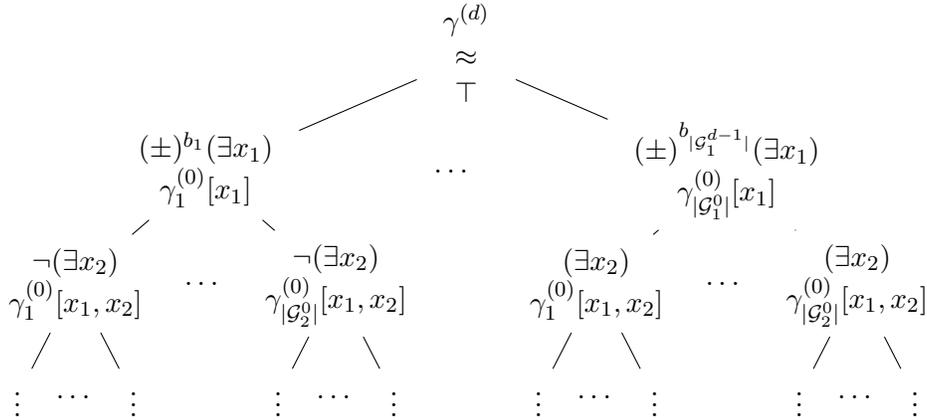
\begin{figure}
    \centering
    \begin{tikzpicture}[
  level 1/.style={sibling distance=18em},
  level 2/.style={sibling distance=9em},
  level 3/.style={sibling distance=4em}]
\node [] (z){\begin{tabular}{c}
      $\mvaconst^{(d)}$ \\
      $\approx$ \\
      $\true$
    \end{tabular}}
  child {node [] (a) {\begin{tabular}{c}
      $(\pm)^{b_1} (\exists x_1)$ \\
      $\mvaconst^{(0)}_{1}[x_1]$
  \end{tabular}}
    child {node [] (b) {\begin{tabular}{c}
      $\lnot (\exists x_2)$ \\
      $\mvaconst^{(0)}_{1}[x_1, x_2]$
  \end{tabular}}
      child {node (c) {$\vdots$}} 
      child {node (d) {$\vdots$}}
    }
    child {node [] (g) {\begin{tabular}{c}
      $\lnot (\exists x_2)$ \\
      $\mvaconst^{(0)}_{\lvert \cG^0_2 \rvert}[x_1, x_2]$
  \end{tabular}}
      child {node (e) {$\vdots$}}
      child {node (f) {$\vdots$}}
    }
  }
  child {node [] (j) {\begin{tabular}{c}
      $(\pm)^{b_{\lvert \cG^{d-1}_1 \rvert}} (\exists x_1)$ \\
      $\mvaconst^{(0)}_{\lvert \cG^0_1 \rvert}[x_1]$
  \end{tabular}}
    child {node [] (bb) {\begin{tabular}{c}
      $(\exists x_2)$ \\
      $\mvaconst^{(0)}_{1}[x_1, x_2]$
  \end{tabular}}
      child {node (cc) {$\vdots$}} 
      child {node (dd) {$\vdots$}}
    }
    child {node [] (gg) {\begin{tabular}{c}
      $(\exists x_2)$ \\
      $\mvaconst^{(0)}_{\lvert \cG^0_2 \rvert}[x_1, x_2]$
  \end{tabular}}
      child {node (ee) {$\vdots$}}
      child {node (ff) {$\vdots$}}
    }
};
\path (a) -- (j) node [midway] {\dots};
\path (b) -- (g) node [midway] {\dots};
\path (bb) -- (gg) node [midway] {\dots};
\path (c) -- (d) node [midway] {\dots};
\path (cc) -- (dd) node [midway] {\dots};
\path (e) -- (f) node [midway] {\dots};
\path (ee) -- (ff) node [midway] {\dots};
\end{tikzpicture}
    \caption{An attributive constituent tree of depth $d$ where nodes are existential formula (except for root node) and edges indicate the scope of the quantifier. Each attributive constituent of depth $d$ corresponds to choosing each $b_i \in \bm{2}$ in the attributive constituent tree of depth $d$. We enumerate every combination of $(\pm)$ from left to right, starting with all negations and moving to all positives.}
    \label{fig:abs:choose:actree}
\end{figure}

For pragmatic purposes, we will need to cleverly choose a filtration. One method for constructing filtrations uses the fact that depth $d$ constituents indicate which depth $d-1$ attributive constituents (with $1$ free individual term) exist or not. As a reminder, constituents are defined in terms of attributive constituents as
\[
\mvconst^{(d)}_{s} = \lAnd_{(r_1, s_1) \in \cG^{d-1}_1} (\pm)^{s(r_1, s_1)} (\exists x_1) \mvaconst^{(d-1)}_{r_1, s_1}[x_1] \,.
\]
Let $o: \cG^{d-1}_1 \rightarrow {\set{O, I}}$ be an \emph{observation} of depth $d$ constituents where $o(r_1, s_1) = O$ means that we observe position $(r_1, s_1)$ and $o(r_1, s_1) = I$ means that we ignore position $(r_1, s_1)$. Then we can define an equivalence class on constituents $\mvconst^{(d)}_s \sim_{o} \mvconst^{(d)}_t$ if $s(r_1, s_1) = t(r_1, s_1)$ whenever $o(r_1, s_1) = O$. The collection of equivalence classes forms a filtration. Whenever $o(r_1, s_1) = I$, we have that every super constituent contains both $\lnot (\exists x_1) \mvaconst^{(d)}_{r', s'}[x_1]$ and $(\exists x_1) \mvaconst^{(d)}_{r', s'}[x_1]$ so that we can no longer tell the two possibilities apart. When $o(r_1, s_1) = I$ for every $(r_1, s_1) \in \cG^{d-1}_1$, the induced filtration produces exactly one super constituent. When $o(r_1, s_1) = O$ for every $(r_1, s_1) \in \cG^{d-1}_1$, the induced filtration assigns each constituent to its own set.

We can further break down the construction of filtrations by constructing an observation of depth $d$ constituents using their substructure. Unfolding the recursive definition of a constituent $\mvconst^{(d)}_s$ by depth, we see that it is a formula of the form
\begin{multline*}
\mvconst^{(d)} = \lAnd_{(r_1, s_1) \in \cG^{d-1}_{1}} (\pm)^{s(r_1, s_1)} (\exists x_1) \mvaconst^{(0)}_{r_1}[x_1] \land \dots \land \lAnd_{r_d \in \cG^0_{d}} (\pm)^{s_{d-1}(r_d)} (\exists x_d) \mvaconst^{(0)}_{r_d}[x_1, \dots, x_d] \,.
\end{multline*}
Figure~\ref{fig:abs:choose:actree} gives an illustration of a depth $d$ attributive constituent tree. In this unfolded form, we see that a depth $d$ attributive constituent is a tree where nodes are existential formulas (except for the root node which is $\true$) of the form $(\pm)^b (\exists x_e) \mvaconst^{(0)}_r[x_1, \dots, x_e]$ and edges indicate the scope of the quantifier. Each partial description $(\pm)^{s(r_1, s_1)} (\exists x_1) \mvaconst^{(d-1)}_{r_1, s_1}[x_1]$ corresponds to a subtree in the attributive constituent of depth $d$ indicating which nested sequences of individuals described by the appropriate depth $0$ attributive constituents exist or not. We can thus construct an observation by indicating which subtrees to observe or ignore.

\subsection{Trailblazer: Game Play with Super Constituents}
\label{subsec:abs:trail}

We can play Pathfinder using super constituents instead of constituents in the obvious way. When Pathfinder is played with super constituents obtained from a filtration that is not a complete filtration, agents will only be able to learn beliefs that enable them to prove a subset of the first-order theorems. This situation makes intuitive sense: we cannot prove certain theorems if we use inappropriate abstractions, even if we have infinite compute. This brings us to a variation of Pathfinder called \emph{Trailblazer} where agents can additionally choose abstractions during game play.

A player is given a depth $d$ super constituent $\mvsconst^{(d)}$ and allowed to make one of three \emph{moves}: select, challenge, or \emph{refine}. The first two are similar to the corresponding ones in Pathfinder. For the last move, a \emph{refine} move takes a super constituent and breaks it into smaller super constituents and chooses one of the smaller super constituents to continue the game. This corresponds to increasing the resolution at which that player would like to continue the game at. We describe the game more formally now.

Let $*$ represent the terminal state reached after a challenge is issued. Let the dependent sum $X \eqdef \sigma_{\cF : \mathscr{F}} (\mathbb{S}_\cF \cup \set{*})$ denote the \emph{states} of Trailblazer which pairs a filtration $\cF$ with the super constituents $\mathbb{S}_\cF$ obtained from filtration $\cF$. Define the \emph{positions} of Trailblazer to be the set
\[
P \eqdef \bigcup_{d \in \N} \seqfunc^d(X) 
\]
of all finite sequences of states.

The turn order for Trailblazer is identical to that of Pathfinder. However, whereas player $O$ selects constituents of odd depth and player $E$ selects constituents of even depth in Pathfinder, this is not the case in Trailblazer due to the refine move.

The transition relation $\rightsquigarrow: P \rightarrow P \rightarrow \bm{2}$ for Trailblazer has an additional clause for refine. For the sake of completeness, we give all the inference rules generating $\rightsquigarrow$ below.
\begin{description}[noitemsep]
    \item[Select] $\seq{\cF_0, \mvsconst^{(0)}} \dots \seq{\cF_d, \mvsconst^{(d)}} \rightsquigarrow \seq{\cF_0, \mvsconst^{(0)}} \dots \seq{\cF_d, \mvsconst^{(d)}} \, \seq{\cF_{d+1}, \mvsconst^{(d + 1)}}$ when $\mvsconst^{(d)} \expandsinto \mvsconst^{(d+1)}$
    \item[Challenge] $\seq{\cF_0, \mvsconst^{(0)}} \dots \seq{\cF_d, \mvsconst^{(d)}} \rightsquigarrow \seq{\cF_0, \mvsconst^{(0)}} \dots \seq{\cF_d, \mvsconst^{(d)}} \, \seq{\cF_d, *}$
    \item[Refine] $\seq{\cF_0, \mvsconst^{(0)}} \dots \seq{\cF_d, \mvsconst^{(d)}} \rightsquigarrow \seq{\cF_0, \mvsconst^{(0)}} \dots \seq{\cF_d, \mvsconst^{(d)}} \, \seq{\cF'_d, \mvsconst^{'(d)}}$ whenever $\cF_d \sqsubseteq \cF'_d$ and $C^{'(d)} \subseteq C^{(d)}$ where $C^{(d)}$ and $C^{'(d)}$ are the cells corresponding to $\mvsconst^{(d)}$ and $\mvsconst^{'(d)}$ respectively
\end{description}
The player whose turn it is to move chooses either select, challenge, or refine.

The winning condition $W: P \rightarrow \set{O, E}$ is the similar to that of Pathfinder where we use inconsistency of super constituents as opposed to inconsistency of constituents.

\begin{definition}
The \emph{Trailblazer} game is given by the tuple $\mathfrak{P} \eqdef (P, T, \rightsquigarrow, W)$.
\end{definition}

As before, there are no winning strategies in Trailblazer and the winning strategy is not computable. Note that we can start game play in Trailblazer with any filtration $\cF$ including the minimal one (\ie, the trivial filtration). Like Pathfinder, Trailblazer is also amenable to self-play training.

\section{Related Work}
\label{sec:rel}

We review related work relevant to each section encountered in the body of the paper. We apologize in advance for missing connections to the literature.

\subsection{Representing Beliefs in Mathematical Statements}
\label{subsec:rel:bel}

As a reminder, the inspiration for the definition of a HT comes from our reading of~\citet[pg. 274--282]{hintikka1970surface}. Our contribution is to extract and formalize some of the ideas for the purposes of ``learning to prove". Notably, we factor out the statics of weight assignment from the dynamics of renormalization as well as formalize renormalization as a Bayesian update localized to subtrees of the refinement tree. To the best of our knowledge, the application of HTs to assigning probabilities to first-order sentences and the embedding of first-order statements are new.

There have been several approaches proposed for assigning probabilities to statements with first-order quantifiers and probabilistic assertions. One approach defines measures on a suitable space of structures where the probability of a statement is the measure of the set of structures that satisfy the statement~\citep[\eg, see][]{gaifman1964concerning, scott1966assigning}~\citep[see][for the case of higher-order logic]{hutter2013probabilities}. Logically valid statements are satisfied in every structure so they are assigned measure $1$. We are not concerned with the ability to express probabilistic assertions in the logic because we simply use the logic to encode mathematics as opposed to empirical propositions. However, we are concerned with weakening the requirement that logically equivalent statements are assigned the same probability. \citet{demski2012logical} proposes another approach (that enforces logical omniscience) that assigns probabilities to sentences based on algorithmic probability.

There have been several approaches developed with learning in mind that assign probabilities to statements based on a measure on structures. A Markov logic network~\citep{richardson2006mln} is a representation designed with probabilistic inference in mind that assigns probabilities to statements expressed in a first-order language interpreted in models with finite domains. The restriction to finite domains means that the setting for Markov logic networks is effectively propositional because an existential quantifier can be encoded as a finite disjunction (similarly, a universal quantifier can be encoded as a finite conjunction). Thus the quantifiers in Markov logic can be translated away at the cost of increasing the sizes of the formulas considered. Blog~\citep{milch2005blog} is a representation that combines first-order logic and probabilities designed with Bayesian inference in mind that assigns probabilities to statements based on a possible worlds semantics.\footnote{Note that this differs from assigning probabilities to possible \emph{kinds} of worlds as we have done which does not directly consider the individuals in the domain of quantification.} Thus the representation also enforces logical omniscience.

\citet{garrabrant2016logical} propose a method called logical induction for assigning probabilities to first-order sentences that only enforces logical omniscience in the limit. Thus their objective of weakening logical omniscience for the purpose of assigning probabilities to mathematical statements is identical to ours although our methods take opposite approaches. Logical induction identifies a market mechanism for assigning probabilities to sentences\footnote{The market mechanism assigns ``prices" which can be interpreted as probabilities using a no Dutch book argument.} and then shows that such a mechanism has nice convergence properties so that logical omniscience holds in the limit but can fail in the finite time regime. In contrast, we begin with a special kind of Bayesian update mechanism that has nice convergence properties by construction and then use it to assign probabilities. Logical omniscience fails provided that we do not start at a fixed point (\ie, a depth HT). 

Another approach to weakening logical omniscience in the context of assigning probabilities to logical sentences is to syntactically model an agent's knowledge. For instance, we can restrict logical omniscience to a subset of sentences~\citep{gaifman2004reasoning} or introducing syntax for a new modality to distinguish implication from provability~\citep{garber1983old}.

It is also possible to adapt a syntactic approach where an agent's reasoning capability is modeled as bounded. For instance,~\citet{icard2014mind} studies bounded reasoning in the context of Markov decision processes and~\citet{bjerring2018dynamic} studies bounded reasoning in the context of bounded-length proofs. Under this approach, we can assign probabilities to statements as the probability of its provability in a proof system where inference rules are applied non-deterministically.\footnote{One issue with this approach is that there can be multiple proofs (or refutations) of the same fact so some notion of minimal length proof is required.}

\subsection{Conjecturing}
\label{subsec:rel:conj}

\citet{larson2005survey} provides a nice survey of the field of automatic conjecture generation. Many of these programs generate conjectures by enumerating syntax (generated from a production system) and pruning them by a combination of hand-crafted heuristics and model checking~\citep[\eg, see][]{lenat1976artificial,fajtlowicz1988conjectures,haase1990invention,colton1999automatic}. Some methods such as the one implemented in Grafitti~\citep{fajtlowicz1988conjectures} are based on the idea of generating the strongest conjecture that has no known counter-example have produced ``research-level" conjectures~\citep[\eg, see][]{chung1988average}. In contrast to these operational descriptions of conjecturing, our description of conjecturing is denotational. One advantage of a denotational approach is that it is not sensitive to the order in which syntax is enumerated.

One form of conjecturing concerns making generalizations from special cases. In short, given that we have seen that $P(a_1), \dots, P(a_N)$ where each $a_i$ is a constant that identifies a unique element in the domain of quantification and $P$ is a unary predicate, to what degree do we believe that $(\forall x) P(x)$ is true? This form of conjecturing has been studied in inductive philosophy. For example,~\cite{carnap1952ind} studies inductive generalizations in monadic first-order logic and~\cite{hintikka1966two} studies inductive generalizations on constituents. We do not address this form of conjecturing. In particular, each $P(a_i)$ results in a depth $0$ dnf whereas $(\forall x) P(x)$ results in a depth $1$ dnf so that we would need to compare conjectures across depth. It would be an interesting direction of future work to analyze the notion of conjecturing while taking depth into account. We note that we can apply any method of inductive generalization defined on constituents~\citep[\eg,][]{hintikka1966two} to our setting.

\subsection{On Games and Proving}
\label{subsec:rel:game}

The connection between games and first-order logic has been recognized since the development of modern first-order logic. The philosopher Peirce casts first-order theorem proving as a non-alternating-turn game on existential graphs~\citep[\eg, see][]{peirce1909m514,sowa2011peirce}. That the semantics of first-order logic can be given in terms of games has also been recognized in the literature~\citep[\eg, see][]{henkin1961inf,hintikka1973logic,lorenz1978gts,hintikka1999game}. The connection between games and other logics (especially modal logic) has also been recognized~\citep[\eg, see][]{benthem2014games}.

There are alternating-turn games that can be played on first-order structures. For instance, the well-known Ehrenfeucht-Fra\"{i}ss\'{e} game, also known as a back-and-forth game, can be used to determine the elementary equivalence of first-order structures. The game-theoretic semantics of first-order logic gives rise to a game for checking the satisfiability of first-order formulas in a given first-order structure, and is an alternating-turn game when played on constituents of the second kind.\footnote{The game semantics of first-order logic is defined by induction on the structure of formulas. It is a game between two players: Eloise who controls the positive fragment of the logic and Abelard who controls the negative fragment of the logic. Negations correspond to switching who controls the positive and negative fragments of the logic. Eloise has a winning strategy if the formula is satisfiable in a structure $\cM$ whereas Abelard has a winning strategy if the formula is not satisfiable in $\cM$.

To see that game play on a dnf results in alternating-turn move order, recall that a constituent $\mvconst^{(d)}[y_1, \dots, y_k]$ of the second kind is a formula of the form $\lAnd (\exists x) \mvconst^{(d-1)}[y_1, \dots, y_k, x] \land (\forall x) \lOr \mvconst^{(d-1)}[y_1, \dots, y_k, x]$. Thus, either (1) Abelard picks a conjunct from $\lAnd$ and passes play to Eloise to instantiate an existential $\exists$ or (2) Abelard instantiates a universal and passes play to Eloise to play a disjunct from $\lOr$. By an induction on $d$, we see that this results in an alternating-turn play. Play begins with Eloise selecting a disjunct.}

\subsection{On Abstraction}
\label{subsec:rel:abs}

\citet{hintikka1970towards} study the concept of definition, a form of abstraction, using constituents. They show a la analysis on constituents that a theory employing definitions that are explicitly definable in a first-order logic can reveal the trivial inconsistency of sentences at shallower depths compared to a theory not employing those definitions. The idea is that definitions are useful, even if they can be translated away, because they make certain theorems easier to prove. In contrast, we consider abstraction as a method for controlling the sizes of constituents, and as a cost, give up the ability to prove certain theorems. 

A form of abstraction, namely proofs with cut (\ie, proofs where we can use lemmas) can be used to reduce the sizes of proofs in first-order proof calculi. Notably, first-order proofs with cut-elimination increases the sizes of proofs by a super-exponential amount~\citep{pudlak1998length}.

\section{Conclusion}
\label{sec:concl}

In summary, we consider the problem of learning a first-order theorem prover where we directly use a representation of beliefs in mathematical claims to construct proofs. Towards this end, we introduce a representation of beliefs that assigns probabilities to the exhaustive and mutually exclusive first-order possibilities found in Hintikka's theory of distributive normal forms. We then examine conjecturing as (statistical) model selection and an alternating-turn proving game that involves determining the consistency of constituents. The game is amenable (in principle) to self-play training for learning beliefs which can be used to construct a prover that is complete when logical omniscience is attained and sound when the beliefs are reasonable. Along the way, we give another method for assigning probabilities to first-order statements that does not enforce logical omniscience as well as an embedding of first-order logic into an associated Hilbert space.

We have left numerous questions unanswered. One direction of future work is to further examine the embedding of first-order logic into a Hilbert space. Another direction of future work concerns the efficient implementation and empirical testing of self-play for Trailblazer (\ie, the variation of the Pathfinder proving game using abstractions). In particular, (1) can we efficiently implement HTs by selecting clever abstractions and using lazy representations, (2) what machine learning representations are effective for representing HTs, and (3) do self-play learning systems for the game learn stable and meaningful evaluation functions that can be used to build actual theorem provers? It is unclear to us how and if these technical issues can be resolved. In spite of the numerous technical difficulties, we are also intrigued by this direction of future work.

It is often said that mathematics is not a spectator sport, that one learns mathematics by doing mathematics. P{\'o}lya expresses this sentiment in conjunction with the necessity of ``plausible reasoning":
\begin{quote}
The result of the mathematician's creative work is demonstrative reasoning, a proof; but the proof is discovered by plausible reasoning, by guessing. If the \emph{learning} [emphasis added] of mathematics reflects to any degree the invention of mathematics, it must have a place for guessing, for plausible inference.~\citep[][pg. vi]{polya1990mathematics1}    
\end{quote}
If we agree with P{\'o}lya, then the implication for applying machine learning to proving is that we require both plausible and demonstrative reasoning in \emph{training} and in \emph{inference}. Put another way, we will be missing many modes of mathematical reasoning that are useful (at least for humans) for the discovery of proofs if we constrain ourselves to an exclusively proof-theoretic view of proving during training. 

What we have accomplished in this paper is largely to give a descriptive account of the mathematical process where one ``learns" mathematics by ``doing" mathematics. More concretely, we have seen that (1) \emph{proving} requires the construction of individuals with certain properties (as opposed to the strategy of eliminating existentials), (2) \emph{conjecturing} can be defined independently of enumerating syntax and can be employed for proving, and (3) \emph{abstractions} (and laziness) are necessary for managing complexity although we (potentially) lose completeness. We hope that the thought experiment conducted in this paper has shed some additional light on the activity of ``learning to prove".

\section*{Acknowledgements}
We thank Dawn Song for helpful discussions and support to pursue this direction of research. We also thank Henryk Michalewski for reading an earlier version of this paper and helpful discussions. Kshitij Bansal, Ian Fischer, Sarah Loos, Marcus Rabe, Christian Szegedy, and Scott Garrabrant provided helpful feedback. Finally, we owe a great deal to Jaakko Hintikka whose work enabled many of the ideas in this paper.

\bibliography{references}

\end{document}